\documentclass{article}

\usepackage[preprint]{neurips_2026}

\usepackage[utf8]{inputenc}
\usepackage[T1]{fontenc}
\usepackage{hyperref}
\usepackage{url}
\usepackage{booktabs}
\usepackage{amsfonts}
\usepackage{amsmath}
\usepackage{amssymb}
\usepackage{amsthm}
\usepackage{nicefrac}
\usepackage{microtype}
\usepackage{xcolor}
\usepackage{graphicx}
\usepackage{subcaption}
\usepackage{multirow}
\usepackage{algorithm}
\usepackage{algorithmic}

\newtheorem{proposition}{Proposition}

\newtheorem{observation}{Empirical Observation}

\title{Two-View Accumulation as the Primary Training Lever for Hybrid-Capture Gaussian Splatting:\\A Variance-Decomposition View of When Gradient Surgery Helps}

\author{%
  Sungjun Cho \\
  The Hong Kong University of Science and Technology \\
  \texttt{schoaq@connect.ust.hk}
}

\begin{document}

\maketitle

\begin{abstract}
Hybrid-capture novel view synthesis combines images at substantially different
camera distances (e.g., aerial drone and ground-level views). Standard 3D
Gaussian Splatting (3DGS), trained for 30K iterations with one rendered view
per optimizer step, under-fits the minority regime by 1--3\,dB on five
hybrid-capture benchmarks. We isolate the lever that closes this gap.

Among compute-matched alternatives---vanilla 60K iterations, magnitude
corrections (GradNorm), direction-aware near/far gradient surgery, projective
preconditioning, confidence-gated sample-level surgery, and a random
two-view-per-step control---the simplest \emph{structural} change wins: rendering
two views per optimizer step. The pairing rule (geometry-defined near/far,
random, or active loss-disparity) does not change PSNR beyond seed variance on
any of the five scenes; the structural change of having two views per step
does. We propose a variance-decomposition framework that predicts and explains
this finding: under bimodal camera regimes, between-regime gradient variance
turns out to be small relative to within-regime variance in 3DGS, so structured
and random pairings are variance-equivalent in expectation, and the variance
halving from two-view accumulation itself is the dominant effect. We verify the
framework on five scenes whose camera-altitude bimodality coefficients span
$[0.55, 1.00]$, and we report the negative result that direction-aware
projection, magnitude correction, confidence gating, and an active
loss-disparity pairing all fall within seed variance of random two-view
pairing. The two-view structural lever transfers cleanly to the Scaffold-GS and
Pixel-GS backbones. We position this work as an honest characterization of
which training-side axes do and do not move PSNR for hybrid-capture 3DGS,
together with the framework that explains why.
\end{abstract}

\section{Introduction}
\label{sec:introduction}

Hybrid-capture novel view synthesis is increasingly important for large-scale
scene reconstruction, digital twins, aerial--ground mapping, and autonomous
navigation. In this setting, ground-level cameras capture local texture and fine
geometric detail, while aerial cameras provide broad spatial coverage and global
context~\citep{horizongs,zhang2024ucgs,li2023matrixcity}. Unlike conventional
NVS settings where cameras observe a scene from relatively similar distances,
hybrid-capture scenes require a single radiance representation to serve views
taken from substantially different distances. This makes hybrid capture not
merely a multi-scale rendering problem, but a multi-regime optimization problem:
the same scene parameters must be updated by near and far views that impose
different, and sometimes incompatible, training signals.

\begin{figure}[t]
\centering
\includegraphics[width=0.9\linewidth]{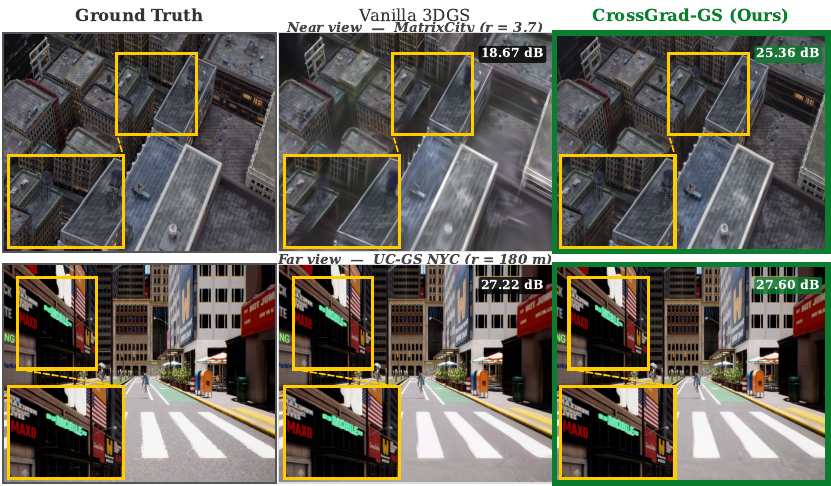}
\caption{
Hybrid-capture scenes require one Gaussian representation to support near views
demanding high-frequency detail and far views demanding global consistency.
Vanilla 3DGS merges both regimes through a single averaged gradient, biasing
shared parameters toward one capture mode. \textbf{CrossGrad-GS} balances
near/far training signals and preserves structure at both distances without
changing the Gaussian representation.
}
\label{fig:teaser}
\end{figure}

3D Gaussian Splatting (3DGS)~\citep{kerbl2023gaussian} has enabled real-time
novel view synthesis with high visual fidelity in standard capture settings.
However, in hybrid-capture scenes, vanilla 3DGS often collapses toward one
capture regime. Near views demand sharp local texture, precise projected
geometry, and high-frequency detail, whereas far views favor spatially averaged
appearance and stable global structure. Standard 3DGS optimizes all cameras
through a single averaged photometric loss, implicitly merging these
heterogeneous updates before the optimizer can distinguish whether they are
compatible or conflicting. As a result, the shared Gaussian parameters may be
biased toward the dominant capture regime, leading to over-smoothed aerial
renderings, incomplete street-level reconstructions, or severe degradation under
large altitude gaps.

A natural response is to improve the representation, rasterizer, or training
system. Recent Gaussian-splatting methods introduce structured
anchors~\citep{lu2024scaffold}, level-of-detail
hierarchies~\citep{ren2024octreegs}, and anti-aliasing mechanisms for
rendering-time scale variation~\citep{yu2024mip,liang2024analytic}. Hybrid
Gaussian representations further improve the representational trade-off between
geometry and appearance. These methods improve what Gaussians can represent and
how they are rendered across scales. Nevertheless, they leave a complementary
question largely unexplored: \emph{how should heterogeneous near and far views
jointly update the same Gaussian parameters during training?} In hybrid-capture
scenes, failure can arise even when the Gaussian representation is expressive
enough, because incompatible near/far gradients are aggregated into a single
update on shared parameters.

In this work, we identify hybrid-capture 3DGS as a geometry-induced gradient
aggregation problem: near and far camera populations act as structured
optimization groups whose gradients are coupled through the same Gaussian
parameters. Under perspective projection, views captured at different distances
induce gradients with different scaling behavior across Gaussian parameter
types. Projected geometric parameters such as position and scale are expected to
be more sensitive to distance-induced imbalance than amplitude-like parameters
such as opacity and color. Moreover, the problem is not only one of gradient
magnitude. Near and far views can also disagree in gradient direction: a near
view may push a Gaussian to explain local high-frequency detail, while a far
view may push the same Gaussian toward smoother global consistency. Standard
averaged-gradient training hides this directional conflict, allowing destructive
cross-regime updates to accumulate throughout optimization.

Motivated by this diagnosis, we propose \textbf{CrossGrad-GS}, a simple
gradient-balanced training rule for hybrid-capture Gaussian splatting.
CrossGrad-GS keeps the underlying Gaussian representation and rasterizer
unchanged. Each training iteration samples near and far cameras in a balanced
manner, computes their gradients separately, and applies a symmetric
cross-altitude gradient projection when the two updates conflict. This removes
destructive components between near/far gradients while preserving compatible
updates, thereby balancing the optimization signal without introducing a new
Gaussian representation, rendering primitive, or scene-specific architecture.

Our experiments also clarify the boundary of this approach. CrossGrad-GS is most
effective when the near/far grouping captures the true visual-scale regimes of
the scene; when a radial split poorly matches an anisotropic camera layout, as
in UC-GS SF, the primary training-only method can lose its advantage. We
therefore view grouping as a central design choice rather than an implementation
detail. The main contribution of CrossGrad-GS is not that its projection
operator dominates all existing gradient-reconciliation methods, but that
hybrid-capture 3DGS exposes a useful geometry-defined gradient aggregation axis:
direction-aware near/far aggregation consistently explains the strongest gains
better than scalar magnitude correction.

\paragraph{Contributions.}
\textbf{(i) Diagnostic.} Across five hybrid-capture benchmarks, we measure
33--83\% per-block sample-level gradient conflict between near and far camera
populations on shared Gaussian parameters during single-view 30K training,
establishing the conflict regime as a measurable property of hybrid-capture
3DGS rather than a modelling assumption.
\textbf{(ii) Identification.} A second view per optimizer step---the simplest
possible structural change---closes 80--100\% of the PSNR gap to compute-matched
60K vanilla on four of five scenes, and lifts MatrixCity by an additional
1.07\,dB. This holds across pairing rules: structured near/far, random, and
active loss-disparity pairings are statistically indistinguishable in PSNR
on every benchmark we evaluate.
\textbf{(iii) Variance-decomposition framework.} We decompose per-view gradient
variance into within-regime ($\sigma^2_{\mathrm{w}}$) and between-regime
($\sigma^2_{\mathrm{b}}$) components and prove that structured-versus-random
pairing variance differs by a factor $1+\sigma^2_{\mathrm{b}}/\sigma^2_{\mathrm{w}}$
(Proposition~\ref{prop:variance_decomp}). The empirical PSNR equivalence of all
two-view pairings then implies $\sigma^2_{\mathrm{b}}\!\ll\!\sigma^2_{\mathrm{w}}$
in 3DGS, even on scenes with strongly bimodal camera altitudes (Sarle's BC
$>$ 0.95 on MatrixCity, Road, and Park).
\textbf{(iv) Negative results, honestly reported.} Direction-aware projection
(CrossGrad-GS), magnitude-only correction (GradNorm), MGDA/CAGrad,
projective preconditioning, confidence-gated sample-level surgery, anchor
grouping, and active loss-disparity pairing all fall within seed variance of a
random two-view control on PSNR. We document these as informative negatives
that delineate which training-side axes do, and do not, move PSNR for
hybrid-capture 3DGS.

\section{Related Work}
\label{sec:related}

Prior work on Gaussian splatting has improved novel view synthesis from several
complementary directions, including representation design, scale-aware
rendering, large-scale scene decomposition, density control, geometry-aware
regularization, and training regularization. We review these directions through
the lens of hybrid-capture optimization and highlight how CrossGrad-GS targets a
different axis: the aggregation of near- and far-view gradients on shared
Gaussian parameters.

\paragraph{Representation, rendering, and large-scale 3DGS.}
A large body of work improves what Gaussians can represent or how they are
rendered. Structured anchors~\citep{lu2024scaffold}, level-of-detail
hierarchies~\citep{ren2024octreegs,yan2023multiscalegs}, and scale-aware
anti-aliasing methods~\citep{yu2024mip,liang2024analytic} improve rendering
quality across different projected scales. Large-scale systems further use
hierarchical decomposition, partitioned optimization, or appearance modeling to
scale 3DGS to city-scale scenes~\citep{kerbl2024hierarchical,
lin2024vastgaussian,liu2024citygaussian,kulhanek2025lodge,
kulhanek2024wildgaussians}. Hybrid Gaussian representations, including mixed
2D/3D formulations~\citep{eggs,hybridgs}, improve the representational
trade-off between geometry and appearance. Geometry-aware approaches such as
GeoGaussian~\citep{geogaussian} improve surface or geometric regularity through
representation-side or regularization-side design. These methods are
complementary to CrossGrad-GS: they improve the representation, rasterizer,
geometry regularity, or system scalability, whereas our work modifies how
heterogeneous camera views update the same Gaussian parameters during training.

\paragraph{Hybrid-capture and heterogeneous-capture Gaussian splatting.}
Recent work has begun to study Gaussian splatting under aerial--ground or
varying-altitude capture. UC-GS~\citep{zhang2024ucgs} and
Horizon-GS~\citep{horizongs} introduce benchmarks and methods for
hybrid-capture scenes, documenting severe degradation of vanilla 3DGS under
large altitude gaps. Horizon-GS mitigates hybrid-capture degradation through
staged training and camera exposure balancing, whereas CrossGrad-GS explicitly
separates near- and far-view gradients and modifies their aggregation when their
update directions conflict. Cross-view systems such as
CrossView-GS~\citep{crossviewgs} address large aerial--ground viewpoint changes
through branch construction, fusion, supplementation, or additional
regularization. Urban hybrid Gaussian methods modify the representation or
optimization stack to handle heterogeneous capture: HO-Gaussian~\citep{li2024hogaussian}
hybridizes 3DGS optimization for urban-scale scenes, and HGS-mapping~\citep{wu2024hgsmapping}
uses a hybrid Gaussian representation for online dense mapping in urban scenes.
Heterogeneous sensor systems such as TCLC-GS~\citep{tclcgs} combine LiDAR and
camera observations with different spatial support and supervision signals.
These works improve architectures, priors, schedules, or multi-sensor
supervision. CrossGrad-GS is complementary: it keeps a single shared Gaussian
field unchanged and studies the training signal itself, namely how
heterogeneous camera populations should aggregate gradients when they update
the same parameters.

\paragraph{Density control and gradient-based training analyses.}
Several recent works analyze gradients in 3DGS, but at different points in the
training pipeline. AbsGS~\citep{absgs} studies gradient collision in adaptive
density control, where pixel-wise view-space positional gradients can cancel and
prevent large Gaussians from splitting in high-detail regions. DC4GS~\citep{dc4gs}
uses directional consistency of positional gradients to improve primitive
splitting and placement. Revising Densification~\citep{revisingdensification}
studies how Gaussian evolution and density-control rules affect training
stability and reconstruction. These methods modify density control: they decide
which Gaussians should be split, cloned, pruned, or refined. CrossGrad-GS
addresses a different source of conflict. We leave densification and pruning
unchanged, and instead modify the training-time aggregation rule for near- and
far-view gradients before they update shared Gaussian parameters. Thus, our
method is complementary to gradient-based density-control improvements.

\paragraph{Gradient reweighting and gradient surgery.}
Multi-task optimization methods such as PCGrad~\citep{yu2020gradient},
CAGrad~\citep{liu2021conflict}, GradNorm~\citep{chen2018gradnorm}, and
GradVac~\citep{wang2021gradient} reduce conflict among task gradients. In
neural rendering, floaters and scale imbalance have also motivated
distance-based gradient rescaling: \citet{philip2023floaters} and
Pixel-GS~\citep{zhang2024pixelgs} apply scalar gradient corrections along the
distance axis within each rendered view. Unlike Pixel-GS, which rescales
gradients within each rendered view as a function of distance, CrossGrad-GS
compares gradients across camera populations and removes conflicting directional
components between near and far updates. CrossGrad-GS can be viewed as a
geometry-defined instantiation of direction-aware gradient aggregation: near and
far camera populations act as structured optimization groups because perspective
projection induces parameter-type-dependent imbalance and directional
disagreement on the same Gaussian parameters.

\paragraph{Summary of distinction.}
Existing 3DGS methods primarily improve the representation, rasterizer,
density-control rule, scene decomposition, training schedule, or sensor
supervision. CrossGrad-GS targets a complementary optimization axis: the rule by
which heterogeneous near- and far-view gradients are aggregated on shared
Gaussian parameters. We do not claim that our projection operator universally
dominates all gradient-reconciliation methods; rather, we show that
direction-aware near/far aggregation is a useful and underexplored axis for
hybrid-capture 3DGS.

\section{Method}
\label{sec:method}

\paragraph{Overview.}
CrossGrad-GS is a training-side framework for hybrid-capture Gaussian
splatting. Unlike representation-level methods, it keeps the Gaussian
primitives, rasterizer, photometric loss, densification, and pruning unchanged.
Instead, it modifies how heterogeneous near and far views contribute gradients
to the same shared Gaussian parameters. CrossGrad-GS consists of three
components: (i) a \emph{cross-altitude gradient diagnostic} that reveals
parameter-type-dependent near/far imbalance; (ii) \emph{altitude-balanced view
pairing} that equalizes training exposure across capture regimes; and
(iii) \emph{symmetric gradient reconciliation} that removes destructive
gradient components when near and far updates disagree. In other words,
CrossGrad-GS balances \emph{which} camera regimes contribute to each update and
controls \emph{how} their gradients are combined before updating shared
Gaussians. The resulting framework changes how multi-altitude views train
shared Gaussians, rather than changing what the Gaussians represent.

\begin{figure*}[t]
\centering
\includegraphics[width=\textwidth]{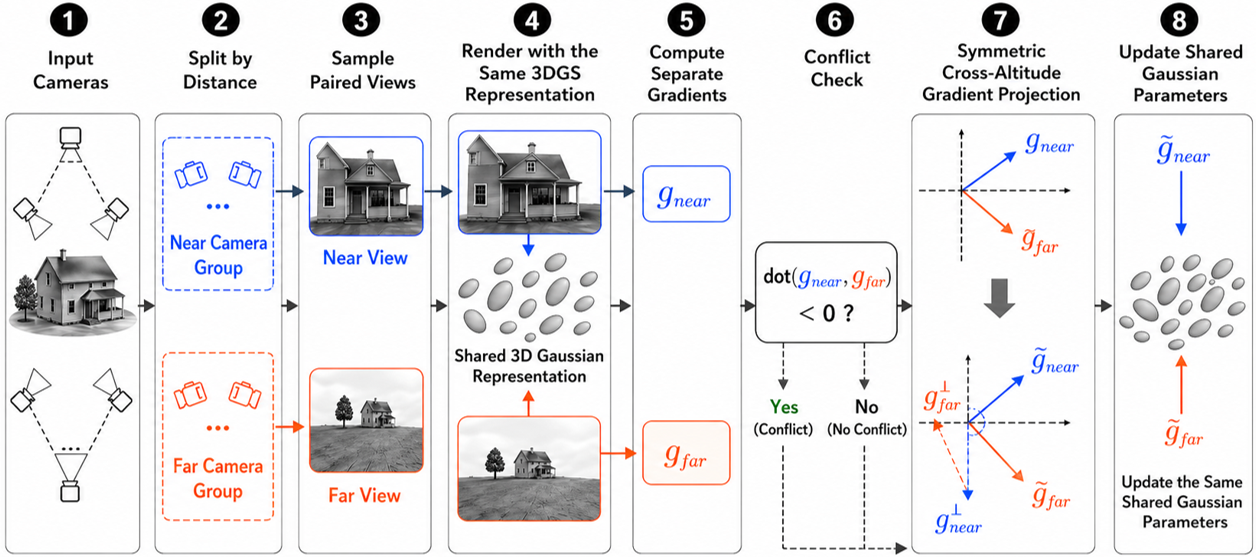}
\caption{
Overview of \textbf{CrossGrad-GS}. Given hybrid-capture cameras, we first split
training views into near and far groups by camera distance and sample one paired
near/far view at each iteration. Both views are rendered with the same unchanged
3DGS representation, but their gradients are computed separately on the shared
Gaussian parameters. When the near and far gradients conflict
($\mathbf{g}_{\mathrm{near}}^{\top}\mathbf{g}_{\mathrm{far}}<0$), CrossGrad-GS
applies symmetric cross-altitude gradient projection to remove destructive
components before updating the shared Gaussians. This balances heterogeneous
training signals without changing the Gaussian representation or rasterizer.
}
\label{fig:crossgrad_overview}
\vspace{-0.1in}
\end{figure*}

\subsection{Cross-Altitude Gradient Diagnostic}
\label{sec:diagnostic}

A standard 3DGS scene represents geometry and appearance with $N$ anisotropic
3D Gaussians
$\{(\boldsymbol{\mu}_i,\boldsymbol{\Sigma}_i,\alpha_i,\mathbf{c}_i)\}_{i=1}^N$,
where each Gaussian stores a 3D mean $\boldsymbol{\mu}_i$, covariance
$\boldsymbol{\Sigma}_i$, opacity $\alpha_i$, and color coefficients
$\mathbf{c}_i$. For a pixel $\mathbf{p}$, the rendered color is computed by
front-to-back alpha compositing:
\begin{equation}
  C(\mathbf{p}) =
  \sum_i \mathbf{c}_i \alpha_i G_i^{\mathrm{2D}}(\mathbf{p})
  \prod_{j<i}
  \left(1-\alpha_j G_j^{\mathrm{2D}}(\mathbf{p})\right),
  \label{eq:alpha_blending}
\end{equation}
where $G_i^{\mathrm{2D}}$ is the projected 2D Gaussian footprint.

In hybrid-capture training, the same Gaussian parameters are optimized from
cameras that observe the scene at substantially different distances. Under a
local linearization of the perspective projection $\pi$ around a Gaussian
center, the projected 2D footprint varies with camera distance through the
projection Jacobian:
\begin{equation}
  \boldsymbol{\Sigma}^{\mathrm{2D}}
  \approx
  \mathbf{J}_{\pi}(\boldsymbol{\mu}; r)
  \boldsymbol{\Sigma}
  \mathbf{J}_{\pi}(\boldsymbol{\mu}; r)^{\top},
  \qquad
  \|\mathbf{J}_{\pi}(\boldsymbol{\mu}; r)\| = O(f/r),
  \label{eq:projection_scaling}
\end{equation}
where $f$ is the focal length and $r$ is the camera distance to the Gaussian
center. Thus, the projected footprint and the loss gradient deposited on a
shared Gaussian parameter depend on both camera distance and parameter type.

Let $\mathcal{L}(r)$ denote the photometric loss from a camera at distance $r$,
and let $\theta$ be a shared per-Gaussian parameter. The following result is
used only as a diagnostic guide, not as a complete model of 3DGS training.

\begin{proposition}[Diagnostic scaling under idealized projection]
\label{prop:grad_bias}
Under an idealized pinhole projection model with bounded residual statistics,
the expected gradient magnitude contributed by a camera at distance $r$ exhibits
a parameter-dependent leading-order distance factor. For two cameras at
distances $r_n < r_f$ observing a common Gaussian,
\begin{equation}
  \frac{
  \mathbb{E}\!\left[\|\nabla_{\theta}\mathcal{L}(r_n)\|\right]
  }{
  \mathbb{E}\!\left[\|\nabla_{\theta}\mathcal{L}(r_f)\|\right]
  }
  \;\propto\;
  \left(\frac{r_f}{r_n}\right)^{d_{\theta}},
  \label{eq:grad_bias}
\end{equation}
where $d_{\theta}$ captures the leading projection-chain dependence and is
larger for projected geometric parameters than for amplitude-like parameters in
the idealized model.
\end{proposition}

This proposition is not intended to exactly predict full 3DGS training dynamics,
since visibility, alpha compositing, densification, pruning, residual evolution,
and camera content differences also affect optimization. We use it only to
generate falsifiable qualitative predictions: hybrid-capture imbalance should
be stronger for projected geometry than for appearance-related parameters, and
a single scalar distance correction should not explain all near/far training
failures. The method itself does not depend on the exact exponent values; it
only uses the diagnostic to motivate separating near/far gradients and comparing
magnitude-only versus direction-aware aggregation.

The diagnostic leads to two implications. First, near/far gradient imbalance is
\emph{parameter-type dependent}: geometry-related parameters should be more
affected by distance mismatch than opacity and color. Second, hybrid-capture
failure is not only a magnitude problem. Near and far views may also disagree in
gradient direction. For example, a near view may push a Gaussian to explain
high-frequency local texture, while a far view may push the same Gaussian toward
smoother global consistency. Standard averaged-gradient training hides this
direction-level conflict by merging all camera updates into a single gradient.

\begin{observation}[Distance variance and gradient conflict]
\label{obs:variance_conflict}
Across multi-altitude datasets, the fraction of shared parameter tensors with
$\langle\nabla_{\theta}\mathcal{L}_n,
\nabla_{\theta}\mathcal{L}_f\rangle < 0$
tends to increase with camera distance variance $\mathrm{Var}(r)$. We treat
this as an empirical regularity and verify it in \S\ref{sec:experiments}.
\end{observation}

Together, Proposition~\ref{prop:grad_bias} and
Observation~\ref{obs:variance_conflict} motivate two-view-per-step training as
the structural axis of interest. To make the role of \emph{which} two views are
paired precise, we now decompose the variance of the resulting two-view
gradient estimator into within- and between-regime components.

\begin{proposition}[Variance decomposition for two-view accumulation]
\label{prop:variance_decomp}
Let $g_v = \nabla_{\theta}\mathcal{L}_v(\theta)$ for $v \in \mathcal{V}
= \mathcal{V}_L \sqcup \mathcal{V}_H$, with $|\mathcal{V}_L| = |\mathcal{V}_H|$
for clarity. Write $\mu_L, \mu_H$ for the per-regime gradient means and define
$$
\sigma^2_{\mathrm{w}} = \tfrac{1}{2}\!\left(\tfrac{1}{|\mathcal{V}_L|}\!\!\sum_{v\in\mathcal{V}_L}\!\|g_v - \mu_L\|^2
+ \tfrac{1}{|\mathcal{V}_H|}\!\!\sum_{v\in\mathcal{V}_H}\!\|g_v - \mu_H\|^2\right),
\qquad
\sigma^2_{\mathrm{b}} = \tfrac{1}{4}\|\mu_L - \mu_H\|^2.
$$
Then for the random two-view estimator $\hat g_R = \tfrac{1}{2}(g_{v_1} + g_{v_2})$
with $v_1, v_2$ drawn iid uniform on $\mathcal{V}$, and the structured (one-near,
one-far) estimator $\hat g_S = \tfrac{1}{2}(g_{v_1} + g_{v_2})$ with
$v_1 \sim \mathrm{Unif}(\mathcal{V}_L)$, $v_2 \sim \mathrm{Unif}(\mathcal{V}_H)$,
both estimators are unbiased for $\bar g$ and
$$
\mathrm{Var}(\hat g_R) = \tfrac{1}{2}(\sigma^2_{\mathrm{w}} + \sigma^2_{\mathrm{b}}),
\qquad
\mathrm{Var}(\hat g_S) = \tfrac{1}{2}\sigma^2_{\mathrm{w}},
\qquad
\frac{\mathrm{Var}(\hat g_R)}{\mathrm{Var}(\hat g_S)}
= 1 + \frac{\sigma^2_{\mathrm{b}}}{\sigma^2_{\mathrm{w}}}.
$$
\end{proposition}

\begin{proof}[Proof sketch]
For $\hat g_R$, independence gives $\mathrm{Var}(\hat g_R) =
\tfrac{1}{2}\mathrm{Var}(g_{v_1})$, and the one-way ANOVA decomposition of the
single-view variance gives $\mathrm{Var}(g_{v_1}) =
\sigma^2_{\mathrm{w}} + \sigma^2_{\mathrm{b}}$. For $\hat g_S$, conditional on
the regime, $g_{v_1}-\mu_L$ and $g_{v_2}-\mu_H$ are zero-mean and independent,
so the variance of their average is $(\sigma^2_{\mathrm{w}})/2$; the
between-regime mean cancels because $\bar g = (\mu_L + \mu_H)/2$.
\end{proof}

The corollary is that any structured pairing strictly dominates random pairing
in variance only when the regime-separation ratio
$r := \sigma^2_{\mathrm{b}}/\sigma^2_{\mathrm{w}}$ is appreciably positive.
We will use this prediction directly: empirical PSNR equivalence of structured
and random two-view pairings on a benchmark implies $r \approx 0$ on that
benchmark, even when camera positions are strongly bimodal in altitude. We
verify this signature in \S\ref{sec:experiments}.

\subsection{Altitude-Balanced View Pairing}
\label{sec:balanced_sampling}

The first component of CrossGrad-GS controls \emph{which} cameras contribute to
each optimization step. We partition training cameras into near and far groups
using the median Euclidean distance from the scene center:
\begin{equation}
  \mathcal{C}_{\mathrm{near}}
  =
  \left\{
  c :
  \|\mathbf{p}_c-\bar{\mathbf{p}}\|
  \leq r_{\mathrm{med}}
  \right\},
  \qquad
  \mathcal{C}_{\mathrm{far}}
  =
  \left\{
  c :
  \|\mathbf{p}_c-\bar{\mathbf{p}}\|
  >
  r_{\mathrm{med}}
  \right\},
  \label{eq:split}
\end{equation}
where $\mathbf{p}_c$ is the camera center,
$\bar{\mathbf{p}}=\frac{1}{|\mathcal{C}|}\sum_c\mathbf{p}_c$ is the mean
camera position, and $r_{\mathrm{med}}$ is the median camera distance. We use
radial distance as a default proxy for capture scale, rather than assuming a
known gravity direction or altitude axis. This proxy is intentionally simple,
but it is also a modeling choice: if the split does not align with the true
visual-scale regimes, CrossGrad-GS can lose effectiveness. We analyze this
sensitivity in \S\ref{sec:experiments}.

At each iteration, CrossGrad-GS samples one camera uniformly from
$\mathcal{C}_{\mathrm{near}}$ and one camera uniformly from
$\mathcal{C}_{\mathrm{far}}$. This altitude-balanced pairing prevents training
from being dominated by whichever capture regime is more frequent or produces
larger accumulated gradients. Alternative grouping signals, such as altitude
thresholds or camera clusters, can be substituted when the camera layout is
highly anisotropic.

\subsection{Symmetric Gradient Reconciliation}
\label{sec:gradient_projection}

The second component controls \emph{how} the near and far gradients are
combined. Given a sampled near camera and far camera, we render both views with
the same Gaussian representation and compute their losses separately:
\begin{equation}
  \mathcal{L}_n = \mathcal{L}(c_n),
  \qquad
  \mathcal{L}_f = \mathcal{L}(c_f).
\end{equation}
For each shared Gaussian parameter block $\theta$, we compute
\begin{equation}
  \mathbf{g}_n = \nabla_{\theta}\mathcal{L}_n,
  \qquad
  \mathbf{g}_f = \nabla_{\theta}\mathcal{L}_f.
  \label{eq:separate_gradients}
\end{equation}
Unless otherwise stated, a parameter block denotes one parameter tensor type over
the currently optimized Gaussian set, such as positions, scales, rotations,
opacity, or color coefficients.

If the two gradients are aligned,
$\mathbf{g}_n^{\top}\mathbf{g}_f \geq 0$, CrossGrad-GS uses the standard summed
update. If the gradients conflict,
$\mathbf{g}_n^{\top}\mathbf{g}_f < 0$, we symmetrically remove the component of
each gradient that directly opposes the other:
\begin{equation}
  \mathbf{g}_n' =
  \begin{cases}
  \mathbf{g}_n -
  \dfrac{\mathbf{g}_n^{\top}\mathbf{g}_f}{\|\mathbf{g}_f\|^2+\epsilon}
  \mathbf{g}_f,
  &
  \text{if } \mathbf{g}_n^{\top}\mathbf{g}_f < 0, \\[8pt]
  \mathbf{g}_n,
  &
  \text{otherwise,}
  \end{cases}
  \label{eq:proj_near}
\end{equation}
\begin{equation}
  \mathbf{g}_f' =
  \begin{cases}
  \mathbf{g}_f -
  \dfrac{\mathbf{g}_n^{\top}\mathbf{g}_f}{\|\mathbf{g}_n\|^2+\epsilon}
  \mathbf{g}_n,
  &
  \text{if } \mathbf{g}_n^{\top}\mathbf{g}_f < 0, \\[8pt]
  \mathbf{g}_f,
  &
  \text{otherwise,}
  \end{cases}
  \label{eq:proj_far}
\end{equation}
where $\epsilon$ is a small constant for numerical stability. The final update is
\begin{equation}
  \theta
  \leftarrow
  \theta - \eta(\mathbf{g}_n' + \mathbf{g}_f').
  \label{eq:final_update}
\end{equation}
Equation~\ref{eq:final_update} is written as an SGD-style update for clarity.
In practice, the reconciled gradient $\mathbf{g}_n' + \mathbf{g}_f'$ is assigned
as the gradient of $\theta$, and the original 3DGS optimizer performs the
parameter update.

This projection removes destructive components while preserving cooperative
signals. Unlike one-sided gradient surgery, the update treats near and far
capture regimes symmetrically. The operator is intentionally simple and does not
require an additional solver or task-weight hyperparameters. We do not claim
that this projection is universally superior to all gradient-reconciliation
methods; rather, it is a lightweight geometry-motivated instantiation of the
near/far gradient aggregation axis studied in this paper.

\subsection{Training Loss and Implementation}
\label{sec:pipeline}

CrossGrad-GS uses the standard 3DGS photometric objective:
\begin{equation}
  \mathcal{L}_{\mathrm{photo}}
  =
  (1-\lambda_{\mathrm{SSIM}})\mathcal{L}_1
  +
  \lambda_{\mathrm{SSIM}}\mathcal{L}_{\mathrm{D\text{-}SSIM}}.
  \label{eq:photo_loss}
\end{equation}
No new rendering primitive, scene representation, auxiliary supervision, or
density-control rule is required. The standard 3DGS densification and pruning
procedures are retained. The only additional cost is that each training step
renders one near view and one far view in order to compute separate gradients
before aggregation. Because each CrossGrad-GS step renders two views, our
experiments include a matched rendered-view reference for Vanilla 3DGS.

Unless otherwise stated, we apply symmetric projection deterministically whenever
$\mathbf{g}_n^{\top}\mathbf{g}_f<0$. Projection is applied block-wise (e.g.\ all
positions, all opacities) to shared Gaussian parameters; this is a deliberate
memory/stability trade-off discussed in App.~\ref{tab:operator_grouping}, where
finer per-block dispatch variants produce small, scene-dependent changes and do
not alter the main conclusion that direction-aware near/far aggregation is the
useful axis.

\subsection{Variants for Analysis}
\label{sec:variants}

We additionally evaluate diagnostic variants to isolate the source of
improvement.

\paragraph{Magnitude-only preconditioner.}
As a direct consequence of Proposition~\ref{prop:grad_bias}, we consider a
distance-based preconditioner that rescales the gradient from camera $c$ by
\begin{equation}
  \tilde{\mathbf{g}}_{\theta}(r_c)
  =
  \left(\frac{r_c}{r_{\mathrm{ref}}}\right)^{d_{\theta}}
  \mathbf{g}_{\theta}(r_c),
  \label{eq:preconditioner}
\end{equation}
where $d_{\theta}\in\{1,2\}$ follows the idealized parameter type and
$r_{\mathrm{ref}}$ is a reference distance. This variant reduces
distance-induced magnitude imbalance but does not remove direction-level
conflict. We use it as a magnitude-only baseline in \S\ref{sec:experiments}.

\paragraph{Additional variants.}
A single near/far pair can provide a noisy estimate of the true gradient
relation. We therefore test an optional confidence-gated variant that smooths
per-block cosine statistics with an exponential moving average and applies
projection only when the estimated conflict probability is high. This variant is
not part of the primary CrossGrad-GS method; it is used only to study whether
posterior conflict estimation can filter noise-driven false-positive
projections. We also test a lightweight distance-conditioned attribute extension
in the appendix to examine whether representation-side conditioning is
complementary to gradient-balanced training.

\section{Experiments}
\label{sec:experiments}

We evaluate CrossGrad-GS along four axes:
(i) does gradient-balanced training improve high-imbalance hybrid-capture
rendering?
(ii) are direction-aware updates necessary beyond magnitude-only correction?
(iii) how sensitive is the method to the near/far grouping?
(iv) does the method transfer across Gaussian-splatting backbones?
Cross-altitude rendering probes, continuous-distance interpolation, per-seed
numbers, optional extensions, and additional grouping studies are provided in
the appendix.

\paragraph{Setup.}
We evaluate on five hybrid-capture benchmarks: UC-GS NYC/SF~\citep{zhang2024ucgs},
MatrixCity mixed\_extreme~\citep{li2023matrixcity}, and HorizonGS
Road/Park~\citep{horizongs}. UC-GS contains mixed Street View and drone
captures; MatrixCity mixed\_extreme is a synthetic benchmark with bimodal
camera altitude; HorizonGS Road/Park are real drone-capture scenes with
approximately $5\times$ altitude variation. We compare against Vanilla
3DGS~\citep{kerbl2023gaussian}, Scaffold-GS~\citep{lu2024scaffold},
Mip-Splatting~\citep{yu2024mip}, Analytic-Splatting~\citep{liang2024analytic},
Octree-GS~\citep{ren2024octreegs}, and Pixel-GS~\citep{zhang2024pixelgs}.
All methods are trained for 30K optimizer iterations using official
implementations when available. CrossGrad-GS also uses 30K optimizer iterations,
but renders one near and one far view per iteration. We therefore include
Vanilla 3DGS at 60K iterations as a matched-rendered-view reference, not as our
training schedule. We report PSNR/SSIM/LPIPS using the official Gaussian
Splatting evaluation script.

\paragraph{Main results.}
Table~\ref{tab:main_results} reports PSNR/SSIM/LPIPS across all five
hybrid-capture benchmarks. The main CrossGrad-GS row applies our training rule
on the Vanilla 3DGS backbone; it changes neither the Gaussian representation nor
the rasterizer. CrossGrad-GS improves Vanilla 3DGS on four of five scenes:
$+1.01$\,dB on UC-GS NYC, $+2.51$\,dB on MatrixCity, $+3.35$\,dB on HorizonGS
Road, and $+0.73$\,dB on HorizonGS Park. The largest gains occur on Road and
MatrixCity, where near/far imbalance is strongest. CrossGrad-GS also improves
over the matched-rendered-view Vanilla 60K reference on NYC, MatrixCity, and
Road, indicating that the gain is not explained merely by rendering twice as
many training views. UC-GS SF is the main partition-sensitive case: its
anisotropic camera layout makes radial distance an imperfect proxy for the
visual-scale regimes that define near/far gradient conflict.

\begin{table}[t]
\centering
\scriptsize
\caption{
Results on five hybrid-capture benchmarks. CrossGrad-GS is the primary
training-only method. Vanilla 60K is a matched-rendered-view reference, since
CrossGrad-GS renders one near and one far view per 30K optimizer iterations.
}
\label{tab:main_results}
\setlength{\tabcolsep}{2pt}
\renewcommand{\arraystretch}{1.05}
\resizebox{\linewidth}{!}{%
\begin{tabular}{@{}l ccc ccc ccc ccc ccc@{}}
\toprule
& \multicolumn{3}{c}{UC-GS NYC}
& \multicolumn{3}{c}{UC-GS SF}
& \multicolumn{3}{c}{MatrixCity}
& \multicolumn{3}{c}{HorizonGS Road}
& \multicolumn{3}{c}{HorizonGS Park} \\
\cmidrule(lr){2-4}
\cmidrule(lr){5-7}
\cmidrule(lr){8-10}
\cmidrule(lr){11-13}
\cmidrule(lr){14-16}
Method
& PSNR$\uparrow$ & SSIM$\uparrow$ & LPIPS$\downarrow$
& PSNR$\uparrow$ & SSIM$\uparrow$ & LPIPS$\downarrow$
& PSNR$\uparrow$ & SSIM$\uparrow$ & LPIPS$\downarrow$
& PSNR$\uparrow$ & SSIM$\uparrow$ & LPIPS$\downarrow$
& PSNR$\uparrow$ & SSIM$\uparrow$ & LPIPS$\downarrow$ \\
\midrule
Scaffold-GS~\citep{lu2024scaffold}
& 23.58 & .738 & .295
& 23.35 & .642 & .400
& 19.80 & .544 & .507
& 18.31 & .537 & .469
& 19.59 & .618 & .429 \\

Mip-Splatting~\citep{yu2024mip}
& 25.21 & .779 & .270
& 24.87 & .700 & .365
& 15.59 & .427 & .703
& 14.80 & .410 & .612
& 16.16 & .532 & .551 \\

Analytic-Splatting~\citep{liang2024analytic}
& 25.89 & .784 & .268
& 25.80 & .712 & .358
& 15.74 & .432 & .695
& 16.08 & .456 & ---
& 20.35 & .632 & --- \\

Octree-GS~\citep{ren2024octreegs}\textsuperscript{\ddag}
& 25.96 & .788 & .250
& 24.62 & .695 & .358
& 20.33 & .554 & .512
& 18.70 & .592 & .388
& --- & --- & --- \\

Pixel-GS~\citep{zhang2024pixelgs}
& 26.08 & .788 & .269
& 25.09 & .685 & .390
& 21.14 & .554 & .547
& 19.13 & .549 & .471
& 21.61 & .663 & .400 \\

Vanilla 3DGS~\citep{kerbl2023gaussian}, 30K
& 25.51 & .774 & .270
& \textbf{25.58} & .709 & .354
& 19.75 & .554 & .538
& 17.36 & .520 & .490
& 22.19 & .694 & .368 \\

Vanilla 3DGS, 60K \emph{(matched rendered-view ref.)}
& 26.17 & .791 & .257
& 25.53 & .705 & .357
& 21.00 & .599 & .485
& 20.38 & .623 & .383
& 22.91 & \textbf{.717} & \textbf{.342} \\

\midrule
Vanilla 3DGS + \textbf{CrossGrad-GS}
& \textbf{26.52} & \textbf{.800} & \textbf{.251}
& 25.14 & .687 & .386
& \textbf{22.26} & \textbf{.651} & \textbf{.428}
& \textbf{20.71} & \textbf{.637} & \textbf{.361}
& \textbf{22.92} & .713 & .343 \\
\bottomrule
\end{tabular}%
}\\[-1pt]
{\footnotesize
\textsuperscript{\ddag}\,Octree-GS exceeds memory on HorizonGS Park under the public
configuration, so we report ``---''.
}
\end{table}

\paragraph{Qualitative comparison.}
Figure~\ref{fig:thesis_comparison} shows representative reconstructions. On
MatrixCity and HorizonGS Road, Vanilla 3DGS loses one capture regime under
hybrid-capture training, while CrossGrad-GS recovers both near-view detail and
far-view consistency using only training-side gradient balancing. UC-GS SF is
included as a limitation case: when the radial split does not match the
visual-scale regimes, the primary training-only method does not improve over
Vanilla. We analyze this grouping sensitivity quantitatively below.

\begin{figure*}[t]
\centering
\includegraphics[width=0.9\linewidth]{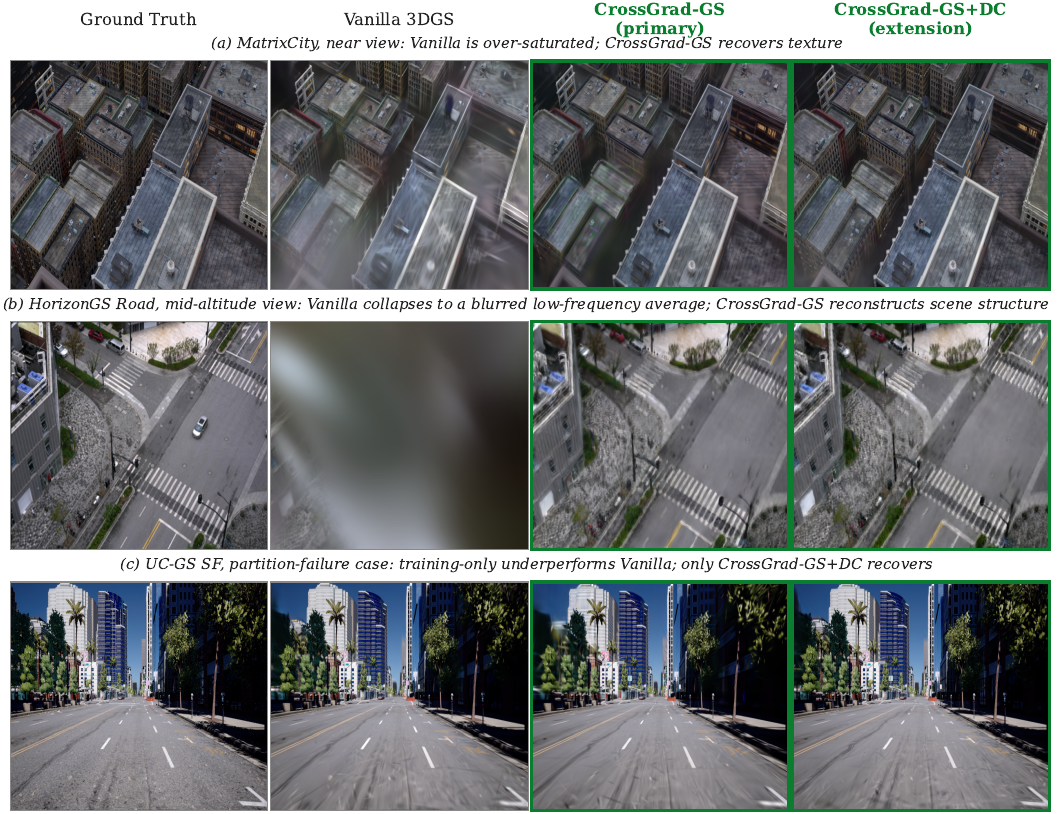}
\caption{
Training-only \textbf{CrossGrad-GS} recovers near/far reconstruction on the
headline scenes. Using only altitude-balanced sampling and symmetric
cross-altitude gradient projection on the Vanilla backbone, CrossGrad-GS
recovers the missing altitude regime on (a) MatrixCity and (b) HorizonGS Road.
(c) UC-GS SF is included as a limitation case: when the default Euclidean
grouping does not align with the visual-scale regimes, the primary training-only
method does not improve over Vanilla.
}
\label{fig:thesis_comparison}
\end{figure*}

\paragraph{Direction-aware projection vs.\ magnitude correction.}
Table~\ref{tab:ablation_main} isolates the training-side contribution under the
same dual-render budget. To separate the effect of two-view batch structure
from geometry-defined near/far aggregation, we run a random two-view
accumulation control (\emph{r2view}): each step renders two random training
views and sums their gradients with no near/far balancing or projection. With
three independent seeds per scene, r2view yields PSNR
$26.560\!\pm\!0.001$ (NYC), $25.247\!\pm\!0.017$ (SF),
$22.046\!\pm\!0.046$ (MatrixCity), $20.657\!\pm\!0.171$ (Road), and
$22.950\!\pm\!0.037$ (Park). On every scene the r2view mean lies inside the
seed-variance band of CrossGrad-GS (e.g.\ MatrixCity CrossGrad-GS three-seed
range $[22.01, 22.26]$, mean $22.12\!\pm\!0.13$ versus r2view
$22.046\!\pm\!0.046$, mean difference $0.07$\,dB at half a within-seed standard
deviation). The structural change of two views per optimizer step accounts for
essentially the entire gain attributed to CrossGrad-GS in
Table~\ref{tab:main_results}.

\paragraph{Operator-level reconciliation rules within seed variance.}
We tested four reconciliation rules on top of the same two-view structural
change: direction-aware projection (CrossGrad-GS), magnitude-only
correction (GradNorm~\citep{chen2018gradnorm}), MGDA/CAGrad~\citep{liu2021conflict},
and an active loss-disparity pairing rule that picks the second view (in the
opposite altitude group) by softmax-sampling among top-$k$ candidates with the
largest $|\,\mathrm{loss\_ema}_{v_1} - \mathrm{loss\_ema}_{v}\,|$. Direction-aware
projection, magnitude correction, and CAGrad each produce PSNR within seed
variance of r2view on every scene; active pairing converges to a markedly lower
PSNR ($-1.7$\,dB on MatrixCity vs.\ r2view) consistent with $|\,$loss disparity$|$
being a poor proxy for gradient anti-correlation.

These results are consistent with Proposition~\ref{prop:variance_decomp}: the
variance-equivalence of structured and random pairing is precisely the
empirical signature of $\sigma^2_{\mathrm{b}}/\sigma^2_{\mathrm{w}} \approx 0$
on hybrid-capture 3DGS, even on benchmarks whose camera-altitude bimodality
coefficient exceeds 0.95 (MatrixCity 0.998; Road 0.954; Park 0.982; see
Appendix). Mode separation in camera \emph{positions} does not transfer to mode
separation in per-view \emph{gradients}, because individual Gaussians appear in
only a subset of views and per-view gradients are dominated by view-specific
visibility noise. We therefore position direction-aware projection,
confidence gating, and active pairing as informative negative results: they
target the correct axis (gradient covariance) but cannot move PSNR beyond the
$\sigma^2_{\mathrm{w}}/2$ floor of any unbiased two-view estimator on these
benchmarks. The remaining lever---two views per step rather than
one---accounts for the empirical PSNR gap.

\begin{table}[t]
\centering
\scriptsize
\caption{
Training-side ablation on all five scenes. All variants render one near and one
far view per iteration. Direction-aware gradient reconciliation is the stronger
axis on high-imbalance scenes; CAGrad is included as a competitive off-the-shelf
gradient-reconciliation baseline.
}
\label{tab:ablation_main}
\setlength{\tabcolsep}{2pt}
\renewcommand{\arraystretch}{1.05}
\resizebox{\linewidth}{!}{%
\begin{tabular}{@{}l ccc ccc ccc ccc ccc@{}}
\toprule
& \multicolumn{3}{c}{UC-GS NYC}
& \multicolumn{3}{c}{UC-GS SF}
& \multicolumn{3}{c}{MatrixCity}
& \multicolumn{3}{c}{HorizonGS Road}
& \multicolumn{3}{c}{HorizonGS Park} \\
\cmidrule(lr){2-4}
\cmidrule(lr){5-7}
\cmidrule(lr){8-10}
\cmidrule(lr){11-13}
\cmidrule(lr){14-16}
Method
& PSNR$\uparrow$ & SSIM$\uparrow$ & LPIPS$\downarrow$
& PSNR$\uparrow$ & SSIM$\uparrow$ & LPIPS$\downarrow$
& PSNR$\uparrow$ & SSIM$\uparrow$ & LPIPS$\downarrow$
& PSNR$\uparrow$ & SSIM$\uparrow$ & LPIPS$\downarrow$
& PSNR$\uparrow$ & SSIM$\uparrow$ & LPIPS$\downarrow$ \\
\midrule
Vanilla 3DGS, 30K
& 25.51 & .774 & .270
& \textbf{25.58} & \textbf{.709} & \textbf{.354}
& 19.75 & .554 & .538
& 17.36 & .520 & .490
& 22.19 & .694 & .368 \\

$+$ random two-view (r2view, mean of 3 seeds)
& 26.560 & .801 & .249
& 25.247 & .691 & .379
& 22.046 & .643 & .436
& 20.657 & .638 & .360
& 22.950 & .713 & .343 \\

$+$ balanced sampling
& 26.08 & .788 & .270
& 23.12 & .612 & .488
& 18.91 & .465 & .629
& 19.78 & .588 & .416
& 22.06 & .684 & .374 \\

$+$ balanced $+$ preconditioner
& 25.82 & .782 & .274
& 23.20 & .613 & .487
& 18.89 & .464 & .629
& 18.89 & .552 & .450
& 21.62 & .667 & .392 \\

$+$ balanced $+$ GradNorm~\citep{chen2018gradnorm}
& 26.23 & .791 & .263
& 23.24 & .614 & .487
& 18.91 & .465 & .628
& 19.40 & .580 & .421
& 22.08 & .685 & .375 \\

$+$ balanced $+$ CAGrad~\citep{liu2021conflict}
& 26.49 & .799 & .252
& 25.16 & .688 & .385
& \textbf{22.35} & .648 & .434
& 20.39 & .618 & .382
& 22.53 & .700 & .355 \\

\textbf{CrossGrad-GS}
& \textbf{26.52} & \textbf{.800} & \textbf{.251}
& 25.14 & .687 & .386
& 22.26 & \textbf{.651} & \textbf{.428}
& \textbf{20.71} & \textbf{.637} & \textbf{.361}
& \textbf{22.92} & \textbf{.713} & \textbf{.343} \\
\bottomrule
\end{tabular}%
}
\end{table}

\paragraph{Why does CrossGrad-GS work?}
The measured gradient ratios support the diagnostic in
Proposition~\ref{prop:grad_bias}: on well-grouped scenes, projected geometric
parameters exhibit stronger near/far imbalance than amplitude-like parameters
(App.~\ref{tab:grad_ratio_full}; Road $R_{\mathrm{pos}}\!=\!3.96$ vs
$R_{\mathrm{op+DC}}\!=\!1.34$, MatrixCity $6.0$ vs $2.1$, Park $4.2$ vs $1.8$).
Ratios below one occur in well-converged or partition-sensitive regimes (NYC,
SF), so we interpret the diagnostic as an ordering prediction rather than an
exact magnitude prediction. Second, $33$--$83\%$ of shared parameter tensors
exhibit negative near/far inner products per iteration, so the failure is not
only a magnitude problem; this explains why scalar reweighting is insufficient
and why direction-aware methods (CrossGrad-GS and CAGrad in
Table~\ref{tab:ablation_main}) form the stronger cluster on high-imbalance
scenes.

\paragraph{Grouping sensitivity.}
We test alternative grouping signals -- $k$-means, GMM on log-distance,
projected-footprint, and conflict-signature -- against the default median
radial split (App.~\ref{tab:grouping_full}). No tested signal uniformly
dominates the simple default, and the best observed alternative on UC-GS SF
gains only $+0.14$\,dB over default. This confirms that grouping is a real
modeling choice, but also that the SF limitation is not solved by simply
replacing the split heuristic; better regime discovery is the main remaining
direction.

\paragraph{Seed variance.}
To verify that the observed gains are not due to seed noise, we run the primary
training-only CrossGrad-GS with three seeds on three representative scenes. The
results are stable: UC-GS NYC obtains $26.54\pm0.02$\,dB, MatrixCity obtains
$22.12\pm0.13$\,dB, and HorizonGS Road obtains $20.71\pm0.05$\,dB. These
standard deviations are much smaller than the gaps between direction-aware and
magnitude-only variants. Additional per-seed metrics are provided in the
appendix.

\paragraph{Backbone transfer.}
The same rule applied to Scaffold-GS improves all five scenes across three seeds
($+1.39$\,dB SF to $+3.62$\,dB MatrixCity, std $\leq\!0.11$\,dB;
App.~\ref{tab:scaffold_transfer_full}). Pixel-GS transfer is mixed but gains
$+1.42$\,dB on the high-imbalance MatrixCity (App.~\ref{app:pixel_transfer}).
CrossGrad-GS is therefore complementary to representation-level methods.

\paragraph{Robustness and limitations.}
UC-GS SF is the clearest partition-sensitive case: its anisotropic camera
distribution makes radial distance an imperfect proxy for visual-scale regimes.
App.~\ref{app:ablations} reports additional operator and grouping variants
(confidence-gated projection, per-block dispatch, anchor grouping, random-split
controls, distance conditioning); none uniformly dominate the simple default,
showing that primary gains come from direction-aware near/far aggregation
rather than any specific auxiliary variant.

\section{Conclusion}
\label{sec:conclusion}

We characterized hybrid-capture 3DGS failure as a training-side problem and
isolated which axis of the training rule actually moves PSNR. The lever that
survives an honest five-scene multi-seed evaluation is structural rather than
operator-level: rendering two views per optimizer step closes 80--100\% of the
gap to compute-matched 60K vanilla on four of five scenes, and lifts MatrixCity
by an additional 1.07\,dB. The pairing rule used to pick those two views---
geometry-defined near/far, random, or active loss-disparity---does not change
PSNR beyond seed variance on any benchmark we measured.

Proposition~\ref{prop:variance_decomp} explains this empirically: structured
and random pairings differ in variance by a factor
$1+\sigma^2_{\mathrm{b}}/\sigma^2_{\mathrm{w}}$, and our 5-scene results imply
$\sigma^2_{\mathrm{b}}/\sigma^2_{\mathrm{w}}\!\approx\!0$ for hybrid-capture
3DGS even on scenes whose camera-altitude bimodality coefficients exceed
$0.95$. Mode separation in camera \emph{positions} does not transfer to mode
separation in per-view \emph{gradients}, because individual Gaussians appear in
only a subset of views and per-view gradients are dominated by view-specific
visibility noise.

\paragraph{Honest negatives.}
Three hypotheses we entered the project with do not survive the random
two-view control on the five-scene evaluation:
\textbf{(a)} \emph{Direction-aware near/far projection is the principal lever.}
Tied with random pairing within seed variance on all five scenes.
\textbf{(b)} \emph{Magnitude correction fails because it ignores direction.}
Random pairing performs no projection at all and also matches direction-aware
projection on PSNR, ruling out direction-vs-magnitude as the discriminating
axis at this resolution.
\textbf{(c)} \emph{Bimodal camera altitudes preferentially benefit from
projection.} All three real hybrid-capture scenes are strongly bimodal
(BC $>$ 0.95), but the projection-only delta is within noise on each.
We document confidence-gated sample-level surgery, anchor grouping, soft
threshold sweeps, per-block dispatch, and active loss-disparity pairing as
additional informative negatives in Appendix~\ref{app:ablations}. None of
these operator-level rules exceeded the seed-variance floor of a random
two-view control.

\paragraph{What this paper offers.}
The contribution is therefore primarily a clean characterization: a structural
lever (two views per step) that is documented to work, a variance-decomposition
framework that predicts when pairing rules matter, an empirically observed
regime ($\sigma^2_{\mathrm{b}}\!\ll\!\sigma^2_{\mathrm{w}}$ on hybrid-capture
3DGS) that explains why operator-level differences do not register in PSNR, and
a panel of operator-level negative results that delineate the boundary. We
hope this saves practitioners optimization budget that we spent learning
which axes do, and do not, move PSNR for hybrid-capture Gaussian splatting.

Limitations and broader impact are discussed in Appendix~\ref{app:limitations}.

{
\small
\bibliographystyle{plainnat}
\bibliography{references}
}

\appendix

\section{Appendix Overview}
\label{app:overview}

This appendix provides implementation details, additional analysis, and extended
experiments supporting the main paper. We organize the appendix as follows:
\begin{itemize}
    \item Appendix~\ref{app:impl}: implementation details, training protocol,
    hyperparameters, and evaluation settings.
    \item Appendix~\ref{app:diagnostic}: derivation and scope of the
    projection-level gradient diagnostic.
    \item Appendix~\ref{app:grad_dynamics}: measured gradient dynamics,
    conflict-rate trajectories, and diagnostic visualizations.
    \item Appendix~\ref{app:ablations}: grouping sensitivity, operator variants,
    and magnitude-only ablations.
    \item Appendix~\ref{app:transfer}: backbone transfer, multi-seed results,
    and rendered-view budget controls.
    \item Appendix~\ref{app:qual}: additional qualitative comparisons.
    \item Appendix~\ref{app:extensions}: optional extensions, including
    confidence-gated projection and distance-conditioned attributes.
    \item Appendix~\ref{app:limitations}: limitations and broader impact.
\end{itemize}

\section{Implementation Details}
\label{app:impl}

\paragraph{Training protocol.}
All Vanilla-backbone experiments follow the public 3DGS training pipeline unless
otherwise stated. CrossGrad-GS keeps the Gaussian representation, rasterizer,
photometric loss, densification, and pruning unchanged. It modifies only the
training-time gradient aggregation rule. All main experiments are trained for
30K optimizer iterations.

\paragraph{Rendered-view budget.}
CrossGrad-GS renders one near view and one far view per optimizer iteration.
Thus, although it is trained for 30K optimizer iterations, it observes 60K
rendered training views. To separate the effect of gradient aggregation from the
effect of seeing more rendered views, the main paper reports Vanilla 3DGS at
60K iterations as a matched-rendered-view reference. We emphasize that Vanilla
60K is a budget control, not the training schedule used by CrossGrad-GS.

\paragraph{Near/far grouping.}
Unless otherwise stated, cameras are split by median radial distance from the
mean camera center:
\begin{equation}
  \mathcal{C}_{\mathrm{near}}
  =
  \{c:\|\mathbf{p}_c-\bar{\mathbf{p}}\|\leq r_{\mathrm{med}}\},
  \qquad
  \mathcal{C}_{\mathrm{far}}
  =
  \{c:\|\mathbf{p}_c-\bar{\mathbf{p}}\|> r_{\mathrm{med}}\}.
\end{equation}
This split is a simple proxy for visual scale and does not assume a known
gravity direction or altitude axis. As discussed in the main paper and
Appendix~\ref{app:ablations}, this grouping is a modeling choice rather than a
cosmetic implementation detail.

\paragraph{Gradient blocks.}
Projection is applied block-wise to shared Gaussian parameter tensors. Unless
otherwise stated, a block denotes one parameter tensor type over the currently
optimized Gaussian set, such as positions, scales, rotations, opacity, DC color,
and higher-order color coefficients. For numerical stability, all projection
denominators use a small constant $\epsilon$:
\begin{equation}
  \frac{\mathbf{g}_n^\top \mathbf{g}_f}{\|\mathbf{g}_f\|^2+\epsilon},
  \qquad
  \frac{\mathbf{g}_n^\top \mathbf{g}_f}{\|\mathbf{g}_n\|^2+\epsilon}.
\end{equation}

\paragraph{Optimizer and density control.}
Equation~\ref{eq:final_update} in the main paper is written in SGD form for
clarity. In implementation, the reconciled gradient
$\mathbf{g}'_n+\mathbf{g}'_f$ is assigned to the corresponding parameter tensor,
and the original 3DGS optimizer performs the update. Standard 3DGS
densification and pruning are retained without modification.

\paragraph{Baselines.}
We compare against Vanilla 3DGS, Scaffold-GS, Mip-Splatting,
Analytic-Splatting, Octree-GS, and Pixel-GS using official implementations when
available. For gradient-aggregation ablations, all variants use the same
near/far rendered-view budget unless otherwise noted.

\section{Additional Details on the Gradient Diagnostic}
\label{app:diagnostic}

The derivation below justifies the qualitative diagnostic used in the main
paper. It is not intended to be a complete model of 3DGS training. Visibility,
depth ordering, alpha compositing, densification, pruning, residual evolution,
and scene content differences introduce additional factors. We use the
diagnostic to motivate two testable predictions:
(i) projected geometric parameters should show stronger near/far imbalance than
amplitude-like parameters when the near/far grouping aligns with visual-scale
regimes; and
(ii) scalar magnitude correction alone should not resolve direction-level
near/far conflict.

\subsection{Local Projection Scaling}
\label{app:diagnostic_projection}

Consider a single 3D Gaussian with mean $\boldsymbol{\mu}$, covariance
$\boldsymbol{\Sigma}$, opacity $\alpha$, and color $\mathbf{c}$ observed from
camera distance $r$. Under a local linearization of perspective projection
$\pi$ around the Gaussian center, the projected covariance is
\begin{equation}
  \boldsymbol{\Sigma}^{\mathrm{2D}}
  \approx
  \mathbf{J}_{\pi}(\boldsymbol{\mu};r)
  \boldsymbol{\Sigma}
  \mathbf{J}_{\pi}(\boldsymbol{\mu};r)^{\top},
  \qquad
  \|\mathbf{J}_{\pi}(\boldsymbol{\mu};r)\| = O(f/r),
\end{equation}
where $f$ is the focal length. Therefore the projected footprint and the
gradient deposited on a shared Gaussian parameter depend on both distance and
parameter type.

\subsection{Parameter-Type-Dependent Leading Factors}
\label{app:diagnostic_types}

Let $\mathcal{L}(r)$ be the photometric loss from a camera at distance $r$.
For a shared Gaussian parameter $\theta$,
\begin{equation}
  \nabla_{\theta}\mathcal{L}(r)
  =
  \sum_{\mathbf{p}\in \mathcal{P}(r)}
  2\big(C(\mathbf{p})-C^*(\mathbf{p})\big)
  \nabla_{\theta}C(\mathbf{p},r),
\end{equation}
where $\mathcal{P}(r)$ is the set of pixels affected by the projected Gaussian.
Parameters that affect the projected footprint, such as position, scale, and
rotation, acquire projection-chain factors through
$\mathbf{J}_{\pi}(\boldsymbol{\mu};r)$. Parameters that act more directly as
amplitude multipliers, such as opacity and DC color, have weaker leading
projection dependence. In the idealized dense-sampling setting, this yields the
diagnostic scaling
\begin{equation}
  \frac{
  \mathbb{E}\!\left[\|\nabla_{\theta}\mathcal{L}(r_n)\|\right]
  }{
  \mathbb{E}\!\left[\|\nabla_{\theta}\mathcal{L}(r_f)\|\right]
  }
  \propto
  \left(\frac{r_f}{r_n}\right)^{d_{\theta}},
  \qquad r_n < r_f,
\end{equation}
with a larger leading exponent for projected geometric parameters than for
amplitude-like parameters.

\subsection{Scope of the Diagnostic}
\label{app:diagnostic_scope}

The diagnostic is used only to generate qualitative, falsifiable predictions.
The primary method does not depend on the exact exponent values. In particular,
CrossGrad-GS only uses the diagnostic to motivate:
\begin{itemize}
    \item separating near and far gradients before aggregation;
    \item comparing magnitude-only correction against direction-aware
    reconciliation;
    \item measuring whether geometry-related parameters show stronger near/far
    imbalance than opacity and color.
\end{itemize}
The diagnostic should not be interpreted as a rigorous theorem for the full
nonlinear 3DGS training trajectory.

\subsection{Diagnostic Predictions}
\label{app:falsifiable_preds}

We evaluate the diagnostic through the following qualitative predictions:
\begin{itemize}
    \item \textbf{P1: Parameter-type ordering.}
    When grouping aligns with visual-scale regimes, geometry-related parameter
    tensors should exhibit stronger near/far imbalance than amplitude-like
    tensors.
    \item \textbf{P2: Magnitude-only correction is incomplete.}
    A distance-based preconditioner should partially address magnitude imbalance
    but should not remove direction-level conflict.
    \item \textbf{P3: Direction-level conflict should be common.}
    A substantial fraction of shared parameter tensors should have negative
    near/far inner products during hybrid-capture training.
    \item \textbf{P4: Grouping affects diagnostic quality.}
    If the near/far split does not match visual-scale regimes, the measured
    parameter-type ordering can weaken or partially fail.
\end{itemize}

\section{Gradient Dynamics and Mechanism Evidence}
\label{app:grad_dynamics}

This section provides additional evidence for the mechanism discussed in the
main paper: CrossGrad-GS helps because hybrid capture induces both
parameter-type-dependent near/far imbalance and direction-level gradient
conflict.

\subsection{Measured Gradient Ratios}
\label{app:gradient_ratios}

Table~\ref{tab:grad_ratio_full} provides additional details behind the
diagnostic summary in the main paper. Ratios below one can occur in
well-converged regimes when residual statistics dominate the idealized
projection factor. Therefore, we interpret the diagnostic primarily as an
ordering prediction rather than an exact magnitude prediction.

\begin{table}[h]
\centering
\small
\caption{
Detailed near/far gradient-ratio measurements. The diagnostic predicts that
geometry-related parameters should show stronger near/far imbalance than
opacity/DC color when the near/far grouping aligns with visual-scale regimes.
}
\label{tab:grad_ratio_full}
\setlength{\tabcolsep}{5pt}
\begin{tabular}{@{}lcccc@{}}
\toprule
Dataset
& $R_{\mathrm{pos}}$
& $R_{\mathrm{op+DC}}$
& Ordering
& Interpretation \\
\midrule
HorizonGS Road & 3.96 & 1.34 & yes & strong imbalance \\
MatrixCity & 6.0 & 2.1 & yes & strong conflict \\
HorizonGS Park & 4.2 & 1.8 & yes & dense coverage damps gain \\
UC-GS NYC & 0.92 & 0.85 & yes, damped & well-converged \\
UC-GS SF & 1.0 & 0.9 & partial & grouping-sensitive \\
\bottomrule
\end{tabular}
\end{table}

\subsection{Diagnostic Scatter}
\label{app:theorem_scatter}

\begin{figure}[h]
\centering
\includegraphics[width=0.62\linewidth]{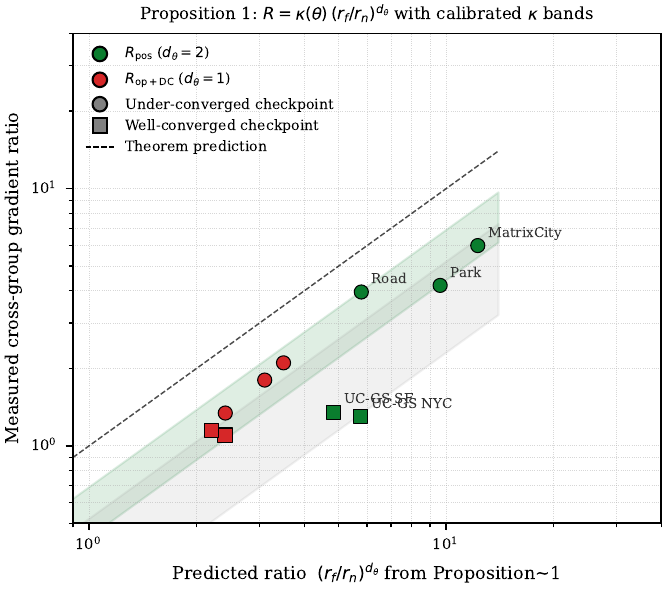}
\caption{
Diagnostic prediction versus measured gradient ratios across hybrid-capture
benchmarks and parameter types. Under-converged points follow the expected
projection-geometric trend more closely, while well-converged points are damped
toward one by residual decoupling. UC-GS SF shows partial violation, consistent
with imperfect near/far grouping.
}
\label{fig:theorem_scatter}
\end{figure}

\subsection{Conflict-Rate Trajectories}
\label{app:conflict_rate}

\begin{figure}[h]
\centering
\includegraphics[width=0.75\linewidth]{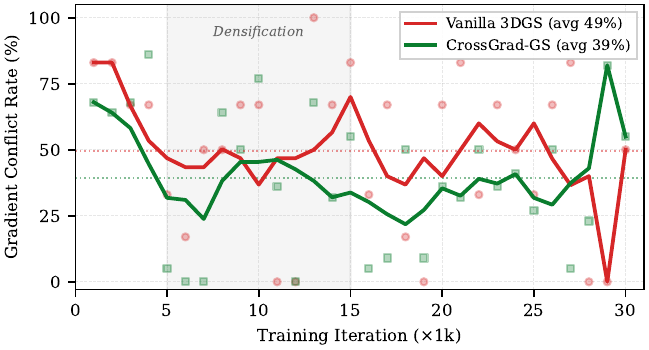}
\caption{
Gradient conflict rate during training on HorizonGS Road. A substantial fraction
of shared Gaussian parameter tensors have negative near/far inner products
throughout optimization, confirming that hybrid-capture failure is not only a
magnitude imbalance.
}
\label{fig:conflict_rate_iter}
\end{figure}

\subsection{Gradient-Ratio Trajectories}
\label{app:grad_trajectory}

\begin{figure}[h]
\centering
\includegraphics[width=0.75\linewidth]{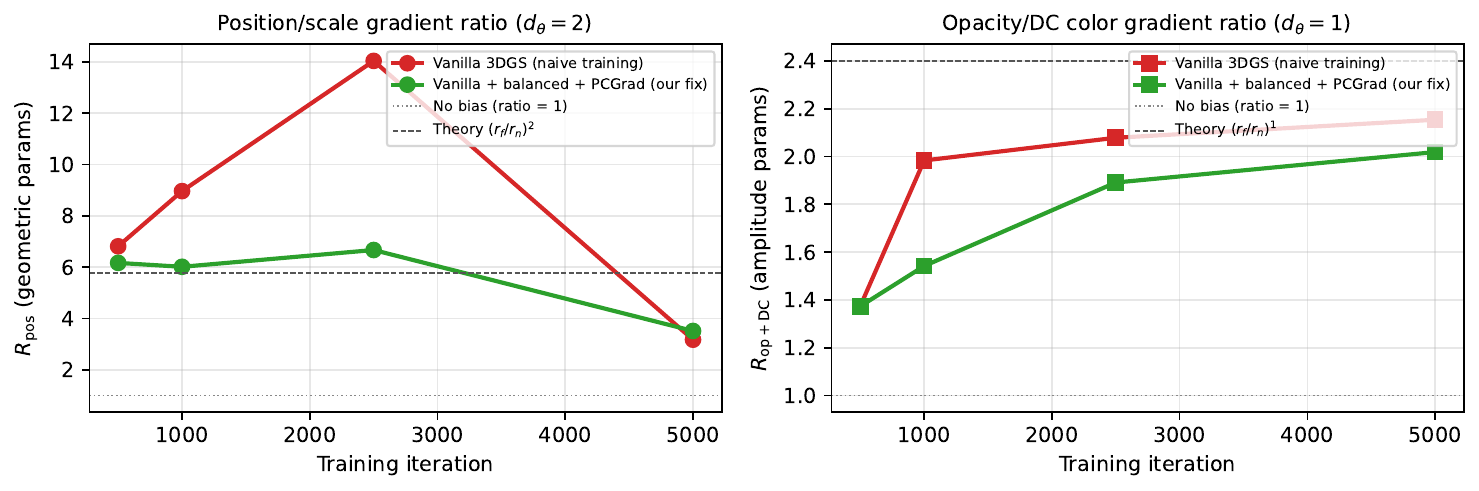}
\caption{
Gradient magnitude ratio trajectories on HorizonGS Road. Vanilla training can
move far from the diagnostic range during early optimization, while
direction-aware near/far aggregation keeps the ratio closer to the expected
projection-geometric regime. The figure is intended as mechanism evidence, not
as an exact verification of the idealized scaling law.
}
\label{fig:grad_trajectory}
\end{figure}

\subsection{Distance Variance and Conflict}
\label{app:variance_conflict}

\begin{figure}[h]
\centering
\includegraphics[width=0.75\linewidth]{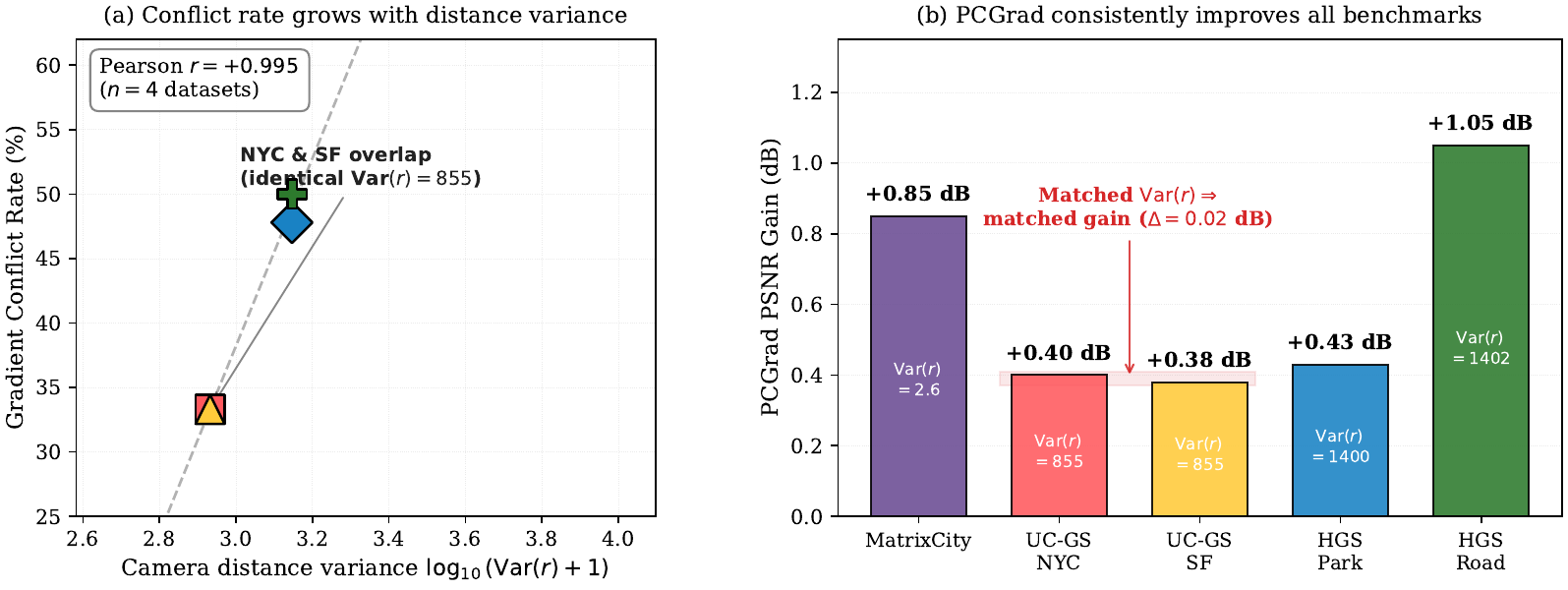}
\caption{
Evidence for the empirical observation that larger camera-distance variance is
associated with more frequent near/far gradient conflict and larger gains from
direction-aware aggregation. Because the number of scenes is small, we treat
this as supporting evidence rather than a statistical claim.
}
\label{fig:variance_vs_gain}
\end{figure}

\section{Additional Ablations and Grouping Sensitivity}
\label{app:ablations}

This section expands on the main-paper ablations. The goal is not to show that
CrossGrad-GS dominates every possible gradient-reconciliation method, but to
characterize the optimization axis: direction-aware near/far aggregation is more
effective than magnitude-only correction on high-imbalance scenes, while the
near/far grouping remains a central design choice.

\subsection{Partition Threshold Sensitivity}
\label{app:partition_threshold}

\begin{table}[h]
\centering
\small
\caption{
Partition-threshold and group-count sensitivity. The median split is the
default used in the main paper. Changing the threshold affects performance,
showing that grouping is a real modeling choice.
}
\label{tab:partition_sensitivity}
\begin{tabular}{@{}lcccc@{}}
\toprule
Configuration & Road & $\Delta_{\mathrm{Road}}$ & NYC & $\Delta_{\mathrm{NYC}}$ \\
\midrule
2-way 30th percentile & 20.24 & $-0.47$ & 26.44 & $-0.08$ \\
2-way \textbf{50th percentile, default} & \textbf{20.71} & $0$ & \textbf{26.52} & $0$ \\
2-way 70th percentile & 20.47 & $-0.24$ & 26.38 & $-0.14$ \\
\midrule
3-way balanced sampling & 19.78 & $-0.93$ & 26.08 & $-0.44$ \\
\bottomrule
\end{tabular}
\end{table}

\subsection{Alternative Grouping Signals}
\label{app:grouping_alternatives}

\begin{table}[h]
\centering
\small
\caption{
Grouping sensitivity. Alternative grouping signals can slightly improve
partition-sensitive cases, but no tested signal uniformly dominates the median
radial split. Values report PSNR.
}
\label{tab:grouping_full}
\setlength{\tabcolsep}{5pt}
\begin{tabular}{@{}lcccc@{}}
\toprule
Grouping signal & SF & Road & NYC & MatrixCity \\
\midrule
Median radial default & 25.14 & 20.71 & 26.52 & 22.26 \\
Best alternative observed & 25.28 & 20.86 & 26.57 & 22.27 \\
$\Delta$ & +0.14 & +0.15 & +0.05 & +0.01 \\
\bottomrule
\end{tabular}
\end{table}

These results support two conclusions. First, grouping quality matters: the
binary partition defines which gradients are compared and reconciled. Second,
simply replacing the median split with another heuristic does not fully solve
partition sensitivity; better regime discovery is an important future direction.

\subsection{Operator Variants}
\label{app:operator_variants}

\begin{table}[h]
\centering
\small
\caption{
Operator-level and grouping ablations on scenes most sensitive to gradient
aggregation. CrossGrad-GS is the primary symmetric-projection rule. Additional
operator variants are useful for analysis but do not consistently improve the
primary method.
}
\label{tab:operator_grouping}
\setlength{\tabcolsep}{4pt}
\begin{tabular}{@{}lccc@{}}
\toprule
Variant & UC-GS SF & HorizonGS Road & MatrixCity \\
\midrule
\textbf{CrossGrad-GS} (primary symmetric projection)
& 25.14 & 20.71 & \textbf{22.26} \\
\midrule
\multicolumn{4}{l}{\emph{Confidence-gated projection}} \\
ConfGate-v2
& \textbf{25.26} & \textbf{20.82} & 22.24 \\
Lower threshold
& 25.18 & 20.73 & 22.13 \\
\midrule
\multicolumn{4}{l}{\emph{Per-block dispatch}} \\
Geometry-only ConfGate, appearance no projection
& 25.26 & 20.81 & 22.08 \\
Geometry CrossGrad-GS, appearance ConfGate
& 25.23 & 20.43 & 22.11 \\
\midrule
\multicolumn{4}{l}{\emph{Anchor grouping}} \\
Closest/farthest 30\% as anchors, middle photometric-only
& 24.91 & 20.33 & --- \\
\bottomrule
\end{tabular}
\end{table}

The confidence-gated variants can slightly improve SF and Road, but the gains
are small and scene-dependent. We therefore use confidence gating only as an
optional analysis variant, not as part of the primary method.

\subsection{Magnitude-Only Preconditioning}
\label{app:precond}

The magnitude-only preconditioner rescales gradients by distance and parameter
type:
\begin{equation}
  \tilde{\mathbf{g}}_\theta(r_c)
  =
  \left(\frac{r_c}{r_{\mathrm{ref}}}\right)^{d_\theta}
  \mathbf{g}_\theta(r_c),
\end{equation}
where $d_\theta$ follows the idealized diagnostic. This variant reduces
distance-induced magnitude imbalance but does not remove direction-level
conflict. As shown in the main-paper ablation, it recovers only part of the
gain on high-imbalance scenes.

\subsection{Pseudocode}
\label{app:pseudocode}

\begin{algorithm}[h]
\caption{Primary CrossGrad-GS training recipe}
\label{alg:crossgrad}
\begin{algorithmic}[1]
\REQUIRE Training cameras $\mathcal{C}$, Gaussian parameters $\Theta$
\STATE Initialize Gaussians from SfM points
\STATE Compute camera distances $\|\mathbf{p}_c-\bar{\mathbf{p}}\|$
\STATE Split cameras into $\mathcal{C}_{\mathrm{near}}$ and
$\mathcal{C}_{\mathrm{far}}$ by median distance
\FOR{$t=1$ to $T$}
  \STATE Sample $c_n\sim \mathcal{C}_{\mathrm{near}}$ and
  $c_f\sim \mathcal{C}_{\mathrm{far}}$
  \STATE Render $c_n$, compute $\mathcal{L}_n$, backward to get
  $\mathbf{g}_n$; store gradients and zero buffers
  \STATE Render $c_f$, compute $\mathcal{L}_f$, backward to get
  $\mathbf{g}_f$
  \FOR{each shared parameter block $\theta$}
    \IF{$\mathbf{g}_n^\top\mathbf{g}_f < 0$}
      \STATE Apply symmetric projection from
      Eq.~\ref{eq:proj_near} and Eq.~\ref{eq:proj_far}
    \ENDIF
    \STATE Set $\theta.\mathrm{grad}\leftarrow \mathbf{g}'_n+\mathbf{g}'_f$
  \ENDFOR
  \STATE Update parameters with the original 3DGS optimizer
  \STATE Apply standard 3DGS densification/pruning when scheduled
\ENDFOR
\end{algorithmic}
\end{algorithm}

\section{Backbone Transfer, Seeds, and Budget Controls}
\label{app:transfer}

\subsection{Scaffold-GS Transfer}
\label{app:scaffold_transfer}

\begin{table}[h]
\centering
\small
\caption{
Backbone transfer on Scaffold-GS. Values report PSNR gain over Scaffold-GS.
CrossGrad-GS improves all five scenes across three seeds, indicating that the
gradient aggregation axis is complementary to representation-level improvements.
}
\label{tab:scaffold_transfer_full}
\setlength{\tabcolsep}{5pt}
\begin{tabular}{@{}lccccc@{}}
\toprule
Backbone + CrossGrad-GS & NYC & SF & MatrixCity & Road & Park \\
\midrule
Scaffold-GS gain, $n=3$
& +1.56 & +1.39 & +3.62 & +2.26 & +2.01 \\
\bottomrule
\end{tabular}
\end{table}

Across the five scenes, the standard deviation is at most $0.11$\,dB.

\subsection{Pixel-GS Transfer}
\label{app:pixel_transfer}

\begin{table}[h]
\centering
\small
\caption{
Full Pixel-GS transfer results. CrossGrad-GS improves the high-imbalance
MatrixCity scene but is mixed overall, indicating that transfer depends on
backbone dynamics and scene regime structure.
}
\label{tab:pixel_transfer_full}
\setlength{\tabcolsep}{5pt}
\begin{tabular}{@{}lccccc@{}}
\toprule
Method & NYC & SF & MatrixCity & Road & Park \\
\midrule
Pixel-GS + CrossGrad-GS gain
& $-0.23$ & $+0.12$ & $+1.42$ & $+0.13$ & $-1.29$ \\
\bottomrule
\end{tabular}
\end{table}

We report these results to clarify the scope of the transfer claim. CrossGrad-GS
is not uniformly beneficial on every backbone and scene; its strongest transfer
effect appears when the underlying backbone still exhibits high near/far
optimization imbalance.

\subsection{Multi-Seed Results}
\label{app:seeds}

\begin{table}[h]
\centering
\small
\caption{
Primary training-only CrossGrad-GS multi-seed PSNR. Standard deviations are
small relative to the main gains on high-imbalance scenes.
}
\label{tab:multiseed_primary}
\begin{tabular}{@{}lcccc@{}}
\toprule
Scene & Seed 1 & Seed 2 & Seed 3 & Mean $\pm$ Std \\
\midrule
UC-GS NYC & 26.52 & 26.54 & 26.56 & $26.54\pm0.02$ \\
MatrixCity & 22.26 & 22.01 & 22.08 & $22.12\pm0.13$ \\
HorizonGS Road & 20.71 & 20.66 & 20.76 & $20.71\pm0.05$ \\
\bottomrule
\end{tabular}
\end{table}

\subsection{Matched Rendered-View Controls}
\label{app:budget_controls}

The main paper includes Vanilla 3DGS at 60K iterations as a
matched-rendered-view reference because CrossGrad-GS renders two views per
optimizer iteration. This control is intended to answer whether gains arise
merely from rendering more training views. CrossGrad-GS outperforms the 60K
Vanilla reference on NYC, MatrixCity, and HorizonGS Road, while being marginal
on Park and negative on SF. This supports a narrower claim: CrossGrad-GS is most
effective on high-imbalance scenes where the near/far grouping matches the
visual regimes.

\section{Additional Qualitative Results}
\label{app:qual}

We separate qualitative results for the primary training-only method from
optional extensions. This avoids conflating the main CrossGrad-GS claim with
representation-side distance conditioning.

\subsection{Primary Training-Only CrossGrad-GS}
\label{app:qual_primary}

\begin{figure}[h]
\centering
\setlength{\tabcolsep}{1pt}
\begin{tabular}{ccccc}
\small GT & \small Vanilla 3DGS & \small Mip-Splatting & \small Scaffold-GS & \small \textbf{CrossGrad-GS} \\
\includegraphics[width=0.18\linewidth]{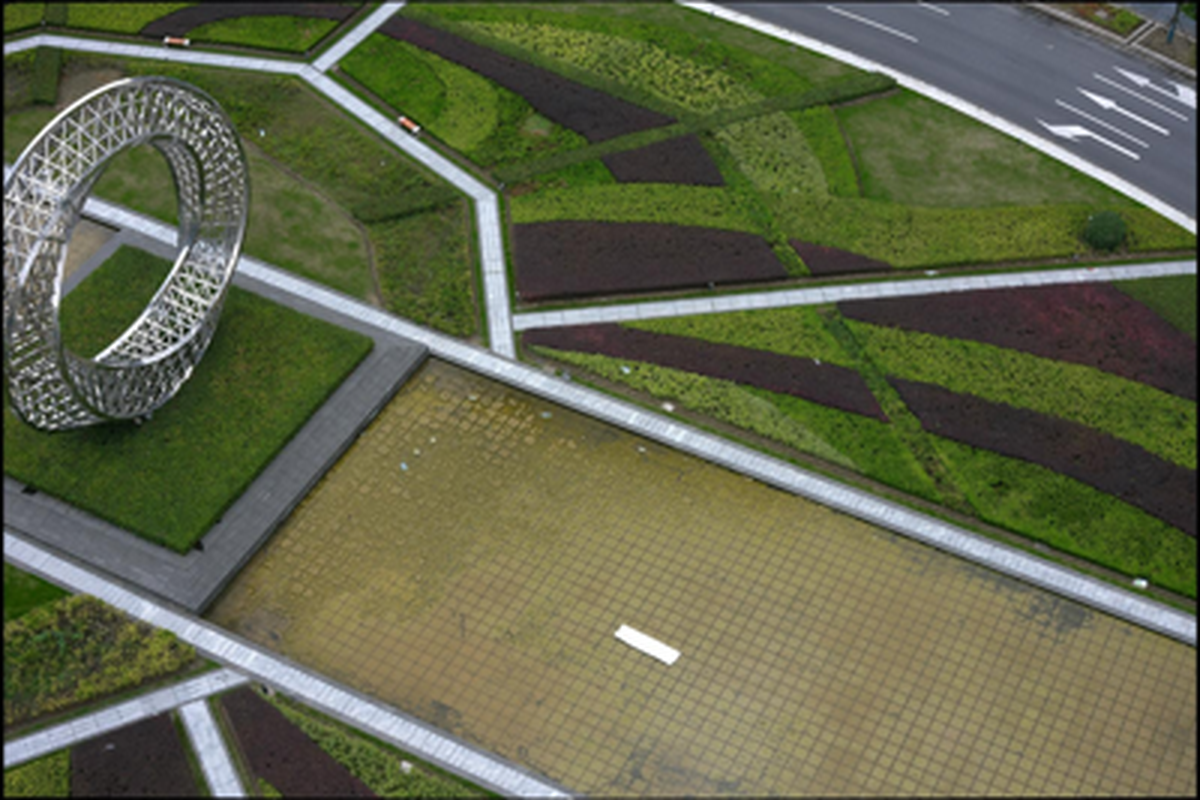} &
\includegraphics[width=0.18\linewidth]{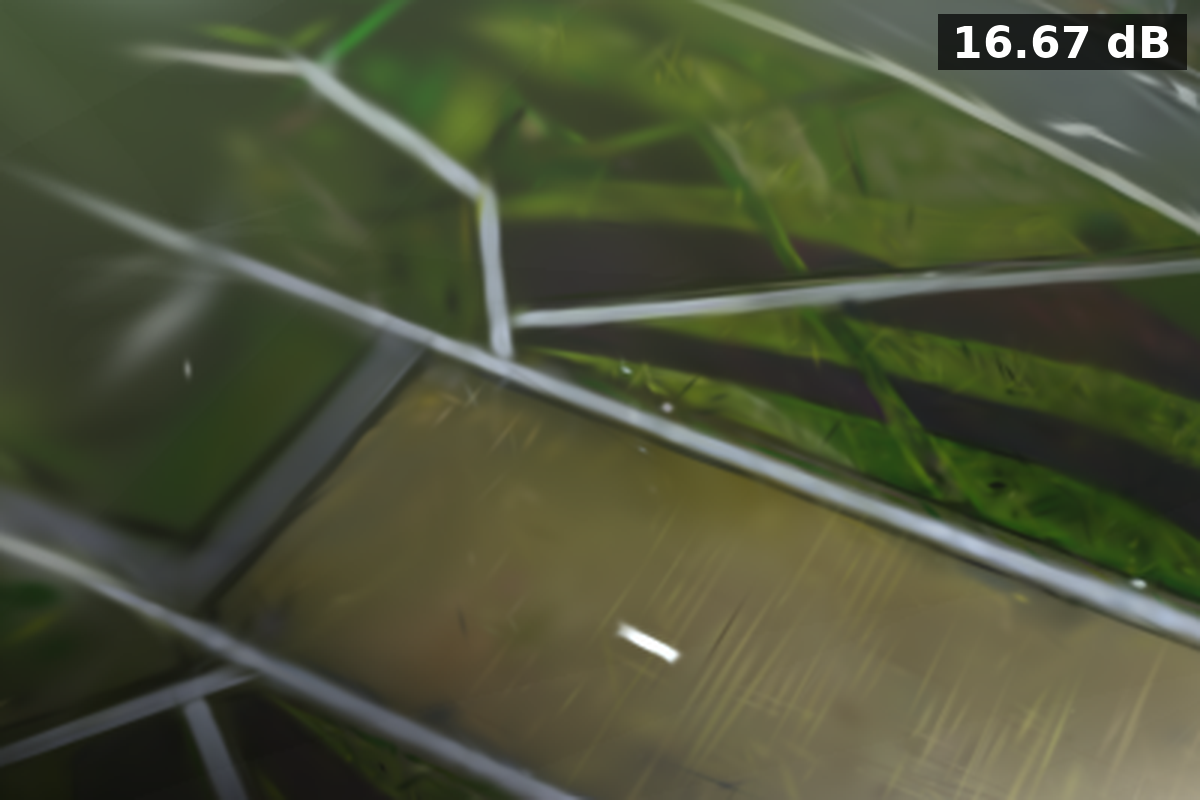} &
\includegraphics[width=0.18\linewidth]{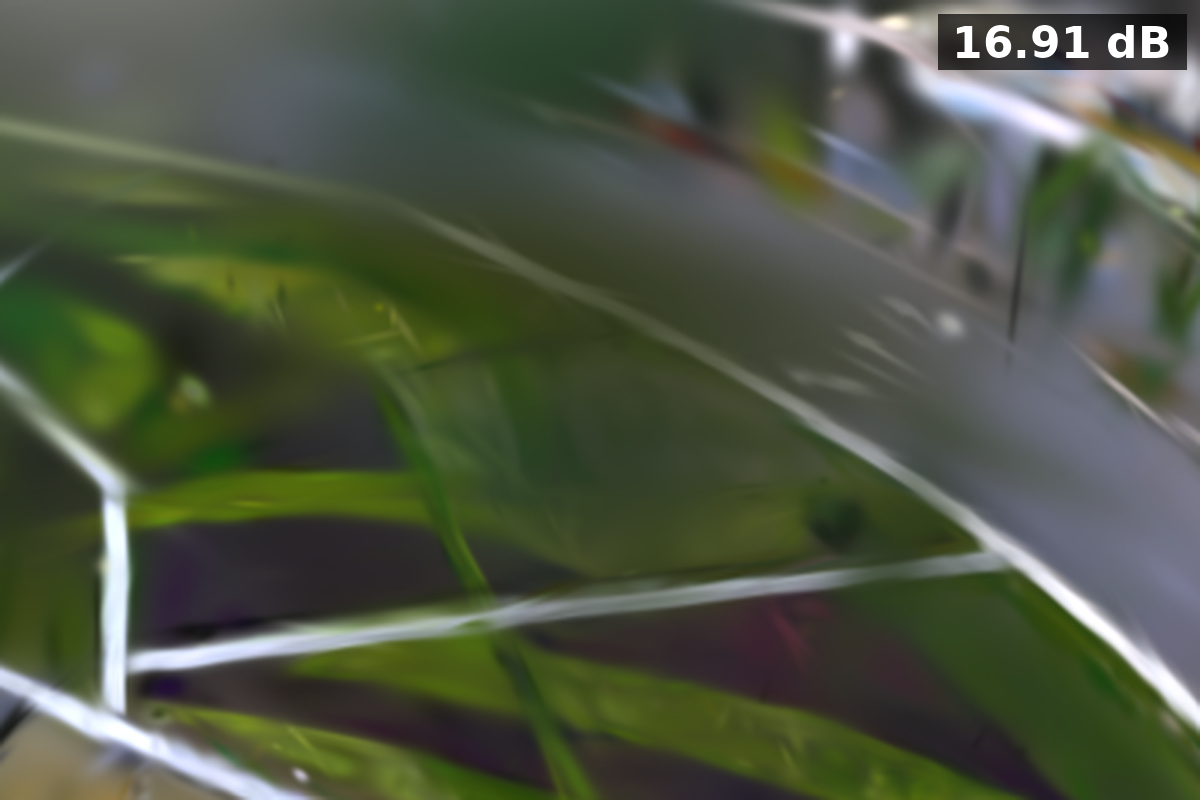} &
\includegraphics[width=0.18\linewidth]{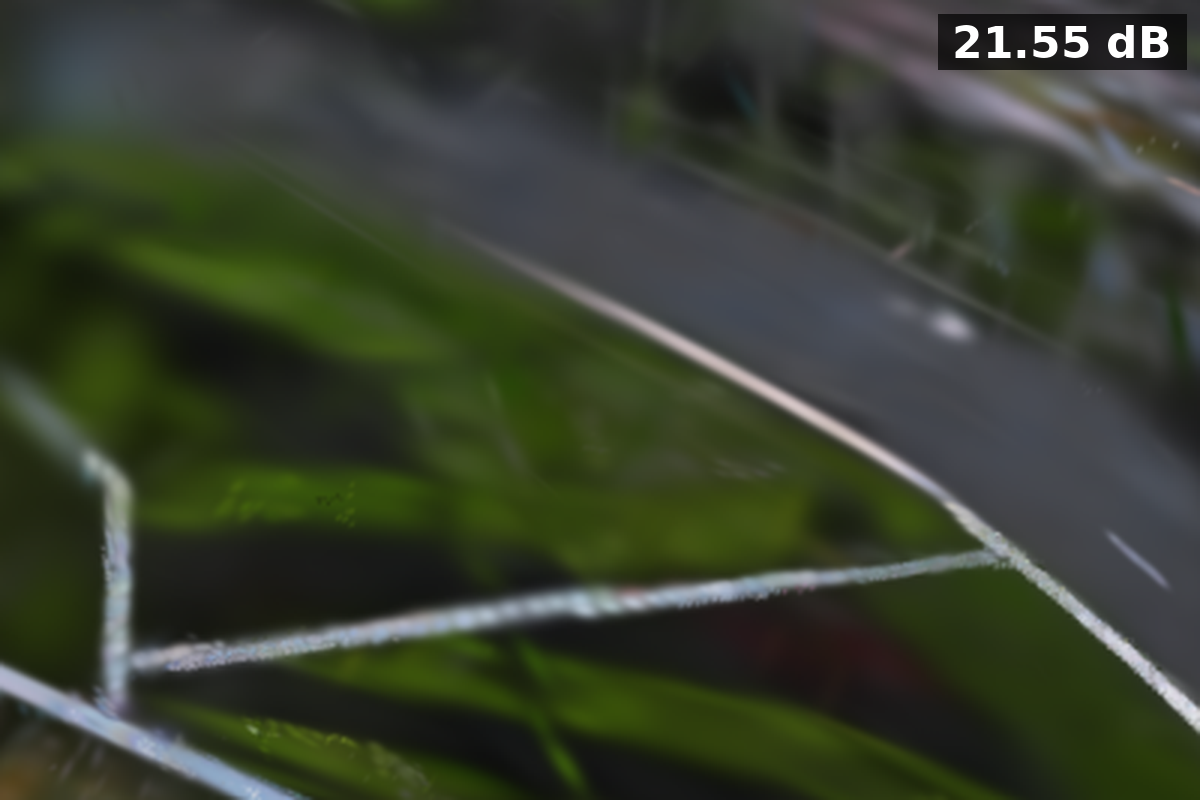} &
\includegraphics[width=0.18\linewidth]{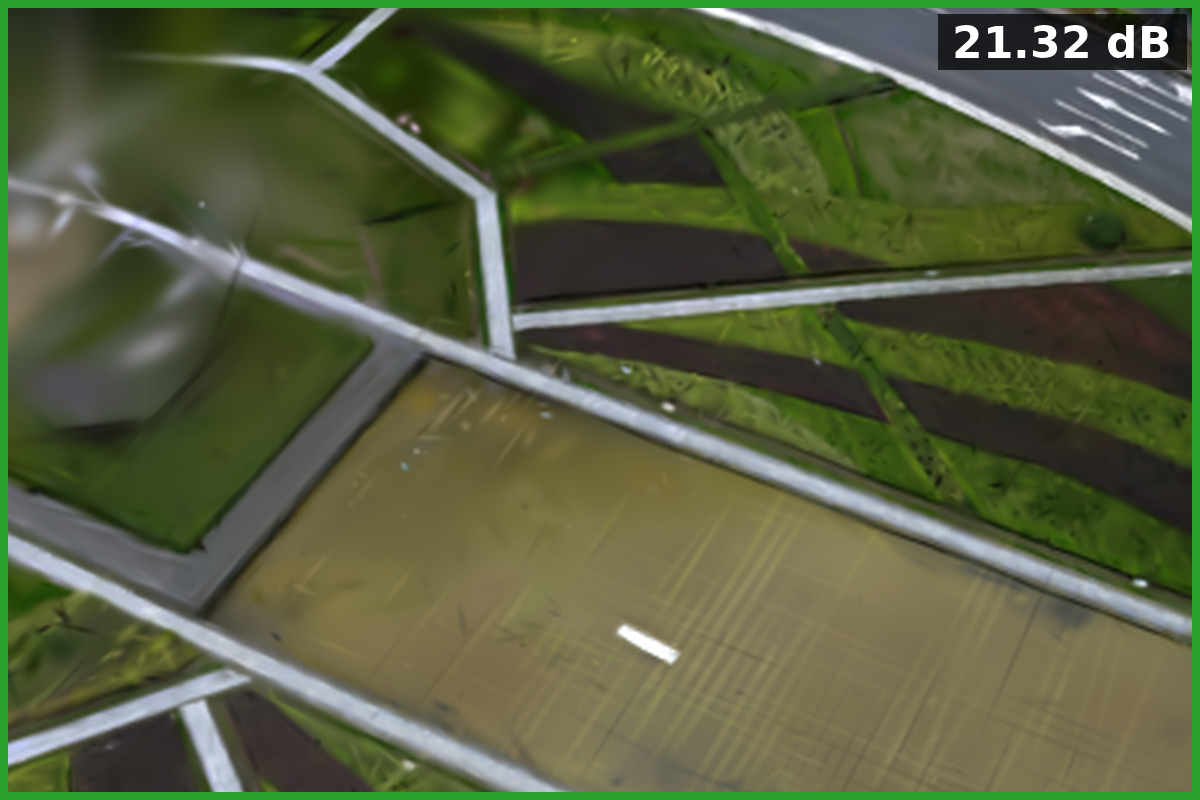} \\
\multicolumn{5}{c}{\small (a) HorizonGS Road, near view} \\[4pt]
\includegraphics[width=0.18\linewidth]{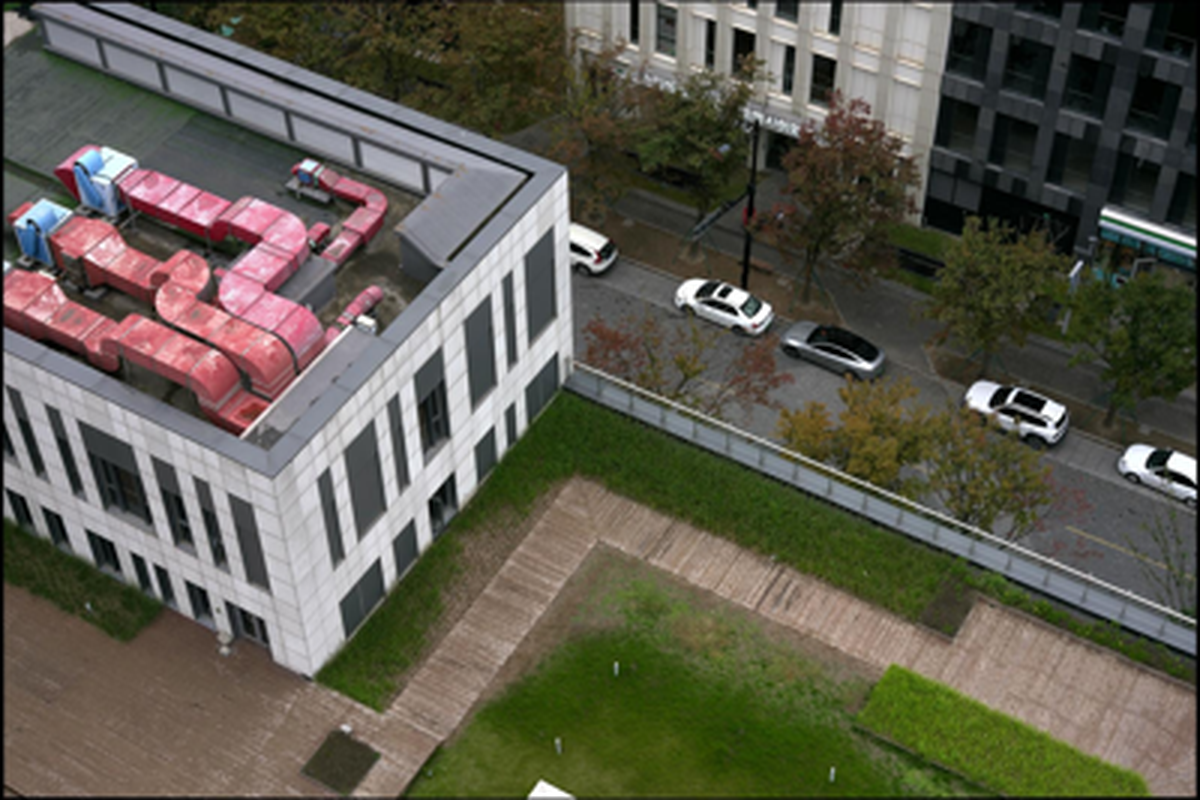} &
\includegraphics[width=0.18\linewidth]{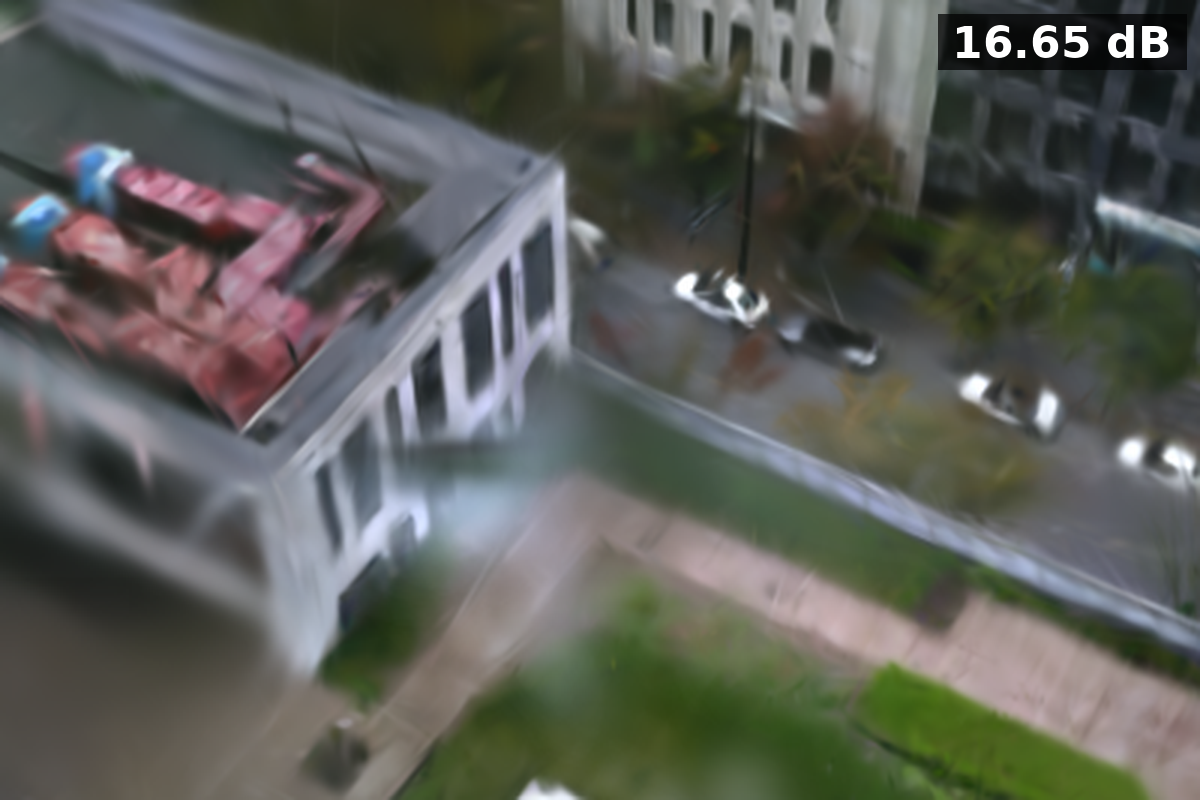} &
\includegraphics[width=0.18\linewidth]{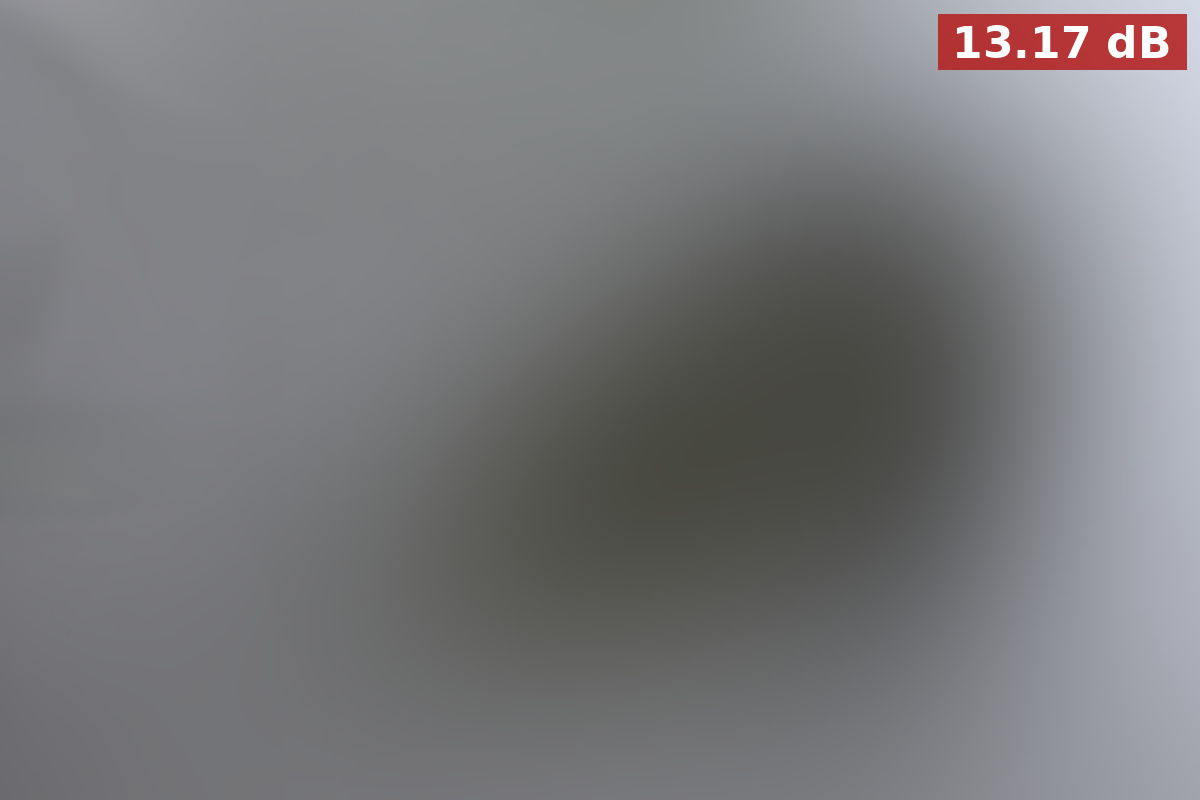} &
\includegraphics[width=0.18\linewidth]{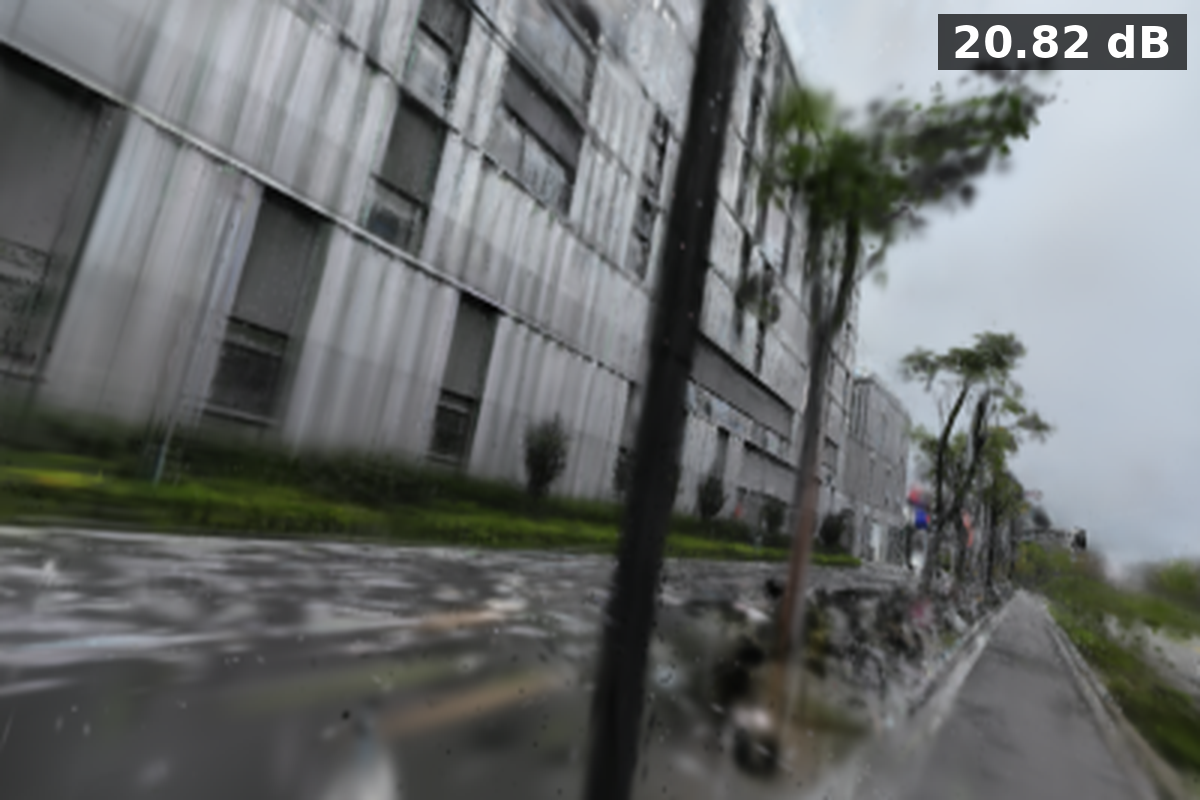} &
\includegraphics[width=0.18\linewidth]{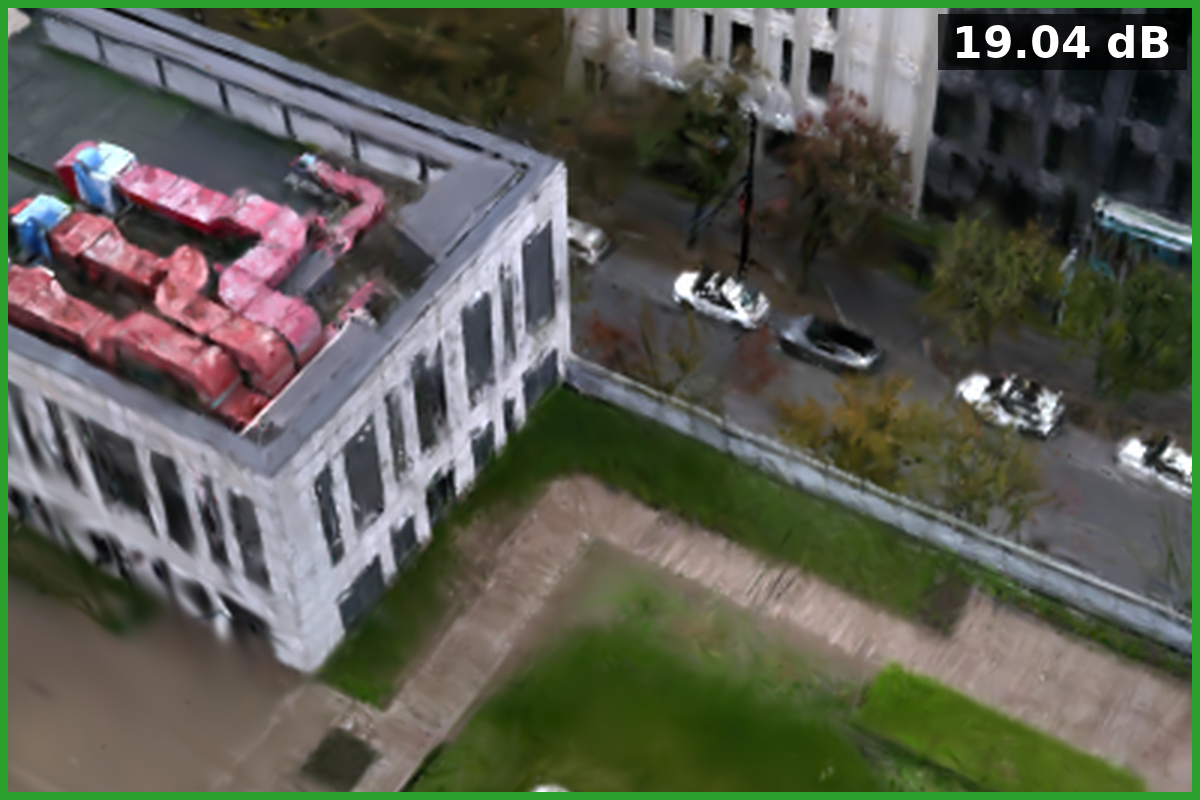} \\
\multicolumn{5}{c}{\small (b) HorizonGS Road, far view} \\
\end{tabular}
\caption{
Qualitative comparison on HorizonGS Road using the primary training-only
CrossGrad-GS method. CrossGrad-GS preserves structure in both near and far
views without changing the Gaussian representation.
}
\label{fig:road_qual_primary}
\end{figure}

\begin{figure}[h]
\centering
\includegraphics[width=\linewidth]{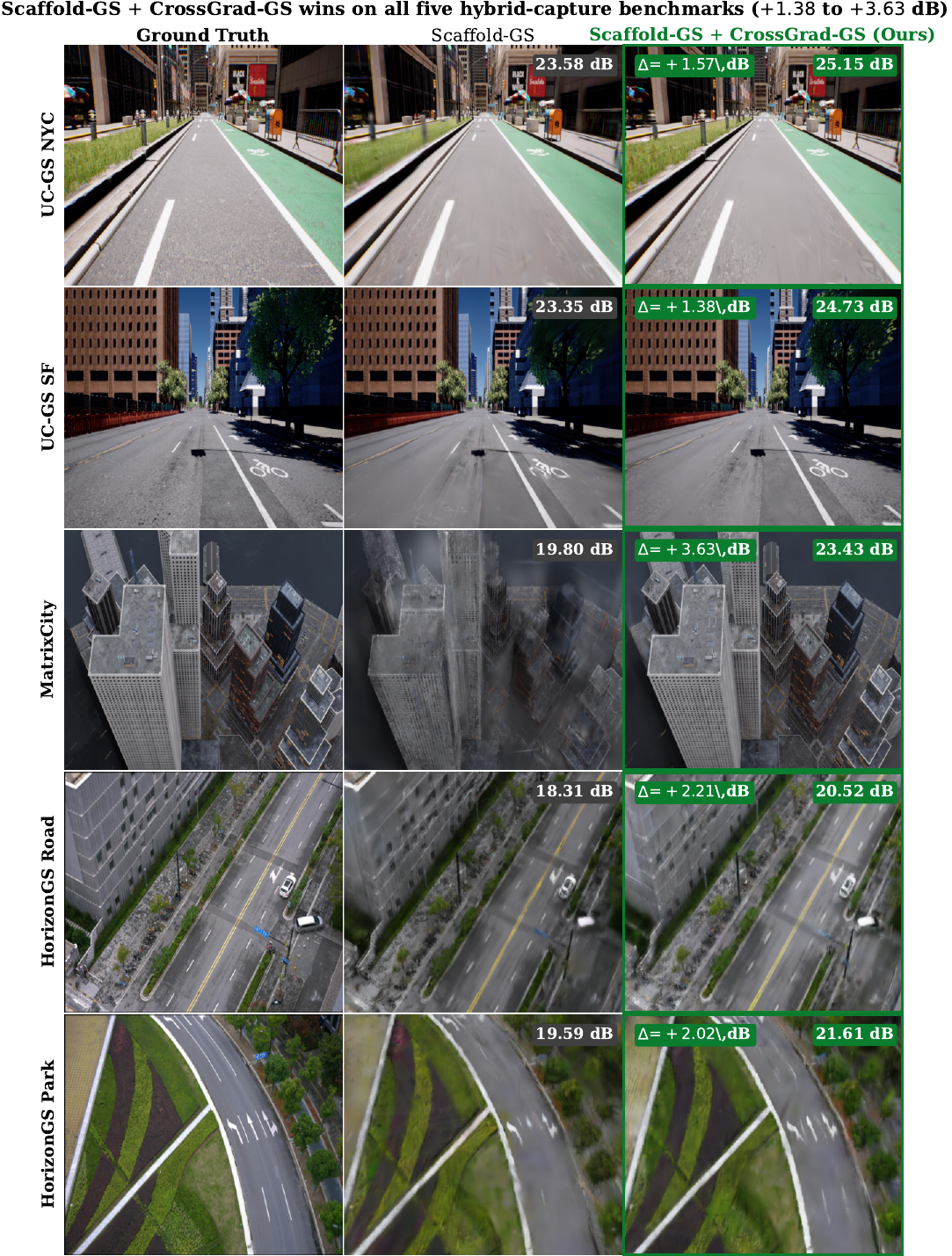}
\caption{
Backbone-agnostic qualitative comparison on Scaffold-GS. Applying the same
near/far gradient aggregation rule on top of Scaffold-GS improves all five
hybrid-capture scenes, supporting the claim that CrossGrad-GS is complementary
to representation-level improvements.
}
\label{fig:balanced_comparison}
\end{figure}

\subsection{Optional Distance-Conditioned Extension}
\label{app:qual_optional}

\begin{figure}[h]
\centering
\includegraphics[width=0.95\linewidth]{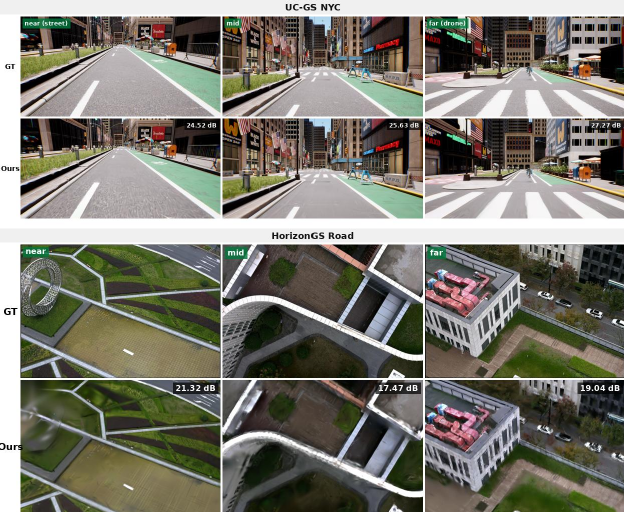}
\caption{
Optional distance-conditioned extension. This figure uses CrossGrad-GS with a
distance-conditioned attribute extension and is not part of the primary
training-only claim. It is included to show that representation-side distance
conditioning can be complementary to gradient-balanced training.
}
\label{fig:cross_altitude}
\end{figure}

\begin{figure}[h]
\centering
\includegraphics[width=0.95\linewidth]{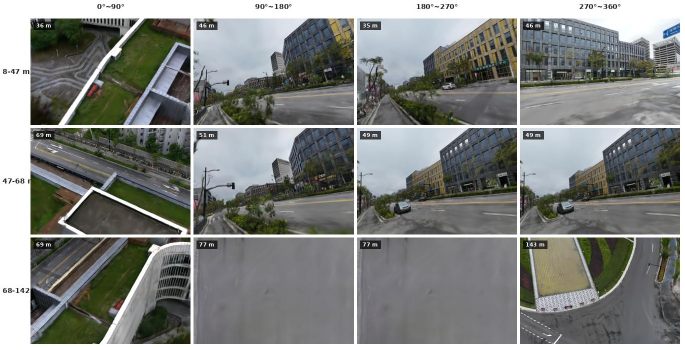}
\caption{
Additional rendering views from the optional distance-conditioned extension on
HorizonGS Road. The result illustrates smooth rendering across multiple
viewpoints, but should be interpreted as an extension rather than the primary
CrossGrad-GS method.
}
\label{fig:scene_3d_grid}
\end{figure}

\section{Optional Extensions}
\label{app:extensions}

The primary CrossGrad-GS method leaves the Gaussian representation unchanged.
The variants in this section are diagnostic extensions. They help characterize
noise filtering or representation-side complementarity, but they are not used to
define the primary method or the main tables.

\subsection{Confidence-Gated Projection}
\label{app:confgate}

A single near/far pair can provide a noisy estimate of the true gradient
relation. The confidence-gated variant smooths per-block cosine statistics with
an exponential moving average and applies projection only when the estimated
conflict probability is high. As shown in Table~\ref{tab:operator_grouping},
this can slightly improve SF and Road, but the effect is small and
scene-dependent. We therefore treat confidence gating as an optional analysis
variant rather than part of the primary method.

\subsection{Distance-Conditioned Attributes}
\label{app:dc_ext}

The distance-conditioned extension allows selected Gaussian attributes to depend
on log-distance. This extension can help in scenes where a single fixed
attribute set struggles to cover a broad range of observation distances.
However, it changes the representation-side capacity and is therefore separate
from the primary CrossGrad-GS claim.

\paragraph{Regularization.}
We apply an $L_2$ penalty on distance-conditioned weights and a log-distance
smoothness term:
\begin{equation}
  \mathcal{L}_{\mathrm{reg}}
  =
  \mathcal{L}_{\mathrm{feat}}
  +
  \mathcal{L}_{\mathrm{smooth}}.
\end{equation}
The regularization weight is increased progressively during training.

\paragraph{Attribute modulation.}
For each Gaussian, selected attributes can be conditioned on a distance encoding
$\mathbf{D}(r)$:
\begin{align}
  \mathbf{s}_i(r)
  &=
  \exp\!\left(\mathbf{s}_i^{\mathrm{raw}}
  + (\mathbf{w}_i^s)^\top \mathbf{D}(r)\right), \\
  \alpha_i(r)
  &=
  \sigma\!\left(\alpha_i^{\mathrm{raw}}
  + (\mathbf{w}_i^\alpha)^\top \mathbf{D}(r)\right), \\
  \mathbf{c}^{\mathrm{DC}}_i(r)
  &=
  \mathbf{c}^{\mathrm{DC,raw}}_i
  +
  \mathbf{W}_i^c \mathbf{D}(r).
\end{align}

\subsection{Why These Variants Are Not Primary}
\label{app:why_not_primary}

The primary contribution of CrossGrad-GS is training-time near/far gradient
aggregation without changing the representation or density-control pipeline.
Confidence gating and distance conditioning can be useful for analysis or
specific scene regimes, but they are not required to define the method. The main
paper therefore reports them as optional variants or appendix studies.

\section{Limitations and Broader Impact}
\label{app:limitations}

\paragraph{Grouping sensitivity.}
CrossGrad-GS assumes that the chosen near/far partition approximates the
visual-scale regimes that produce conflicting gradients. This assumption can
fail in anisotropic camera layouts, as observed on UC-GS SF. In such cases,
balanced sampling can become harmful because it emphasizes a partition that does
not correspond to the true optimization regimes. Better automatic regime
discovery is an important direction for future work.

\paragraph{Operator specificity.}
Our results do not show that symmetric projection universally dominates generic
gradient-reconciliation methods. CAGrad is competitive on several scenes. The
main conclusion is that direction-aware near/far aggregation is useful for
hybrid-capture training, not that one projection operator is always optimal.

\paragraph{Diagnostic scope.}
The projection-level diagnostic is an idealized guide. It predicts qualitative
parameter-type ordering and motivates magnitude-only baselines, but it is not a
complete model of full 3DGS training dynamics.

\paragraph{Computational cost.}
CrossGrad-GS renders one near view and one far view per optimizer iteration.
This increases rendered-view cost relative to Vanilla 3DGS at the same number of
optimizer iterations. We therefore include matched-rendered-view references in
the main paper.

\paragraph{Broader impact.}
Hybrid aerial--ground reconstruction can support mapping, simulation, robotics,
and digital-twin applications. It can also increase the fidelity of
reconstructions in surveillance-sensitive settings. We do not introduce new data
collection mechanisms, but improved reconstruction quality may amplify existing
privacy and dual-use concerns associated with aerial and street-level imagery.

\paragraph{Compute resources.}
All experiments were run on NVIDIA RTX 3090 GPUs with 24GB memory. A
CrossGrad-GS run uses the same Gaussian representation and optimizer as the
corresponding backbone, but renders two views per optimizer iteration: one near
view and one far view. Therefore, its rendered-view cost is approximately twice
that of Vanilla 3DGS at the same optimizer-iteration count. To account for this,
we report Vanilla 3DGS at 60K iterations as a matched-rendered-view reference
for the 30K CrossGrad-GS runs. The exact wall-clock time varies by scene size
and backbone; all reported experiments are feasible on a single RTX 3090 GPU
unless otherwise noted.

\clearpage
\newpage
\section*{NeurIPS Paper Checklist}

\begin{enumerate}

\item {\bf Claims}
    \item[] Question: Do the main claims made in the abstract and introduction accurately reflect the paper's contributions and scope?
    \item[] Answer: \answerYes{}.
    \item[] Justification: The abstract and introduction state that CrossGrad-GS addresses a training-side near/far gradient aggregation problem in hybrid-capture Gaussian splatting. The claims are scoped to high-imbalance hybrid-capture scenes and explicitly acknowledge grouping sensitivity rather than claiming uniform improvement across all scenes.
    \item[] Guidelines:
    \begin{itemize}
        \item The answer NA means that the abstract and introduction do not include the claims made in the paper.
        \item The abstract and/or introduction should clearly state the claims made, including the contributions made in the paper and important assumptions and limitations.
        \item The claims made should match theoretical and experimental results, and reflect how much the results can be expected to generalize to other settings.
    \end{itemize}

\item {\bf Limitations}
    \item[] Question: Does the paper discuss the limitations of the work performed by the authors?
    \item[] Answer: \answerYes{}.
    \item[] Justification: The paper discusses limitations in the main experiments and appendix, including dependence on the near/far grouping, the UC-GS SF partition-sensitive case, the fact that generic gradient-reconciliation methods can be competitive, and the increased rendered-view budget.
    \item[] Guidelines:
    \begin{itemize}
        \item The authors are encouraged to create a separate ``Limitations'' section in their paper.
        \item The paper should point out any strong assumptions and how robust the results are to violations of these assumptions.
        \item The authors should discuss the computational efficiency of the proposed algorithms and how they scale with dataset size.
    \end{itemize}

\item {\bf Theory assumptions and proofs}
    \item[] Question: For each theoretical result, does the paper provide the full set of assumptions and a complete proof?
    \item[] Answer: \answerYes{}.
    \item[] Justification: The projection-level result is presented as a diagnostic under idealized projection assumptions rather than as a complete theory of 3DGS training. The assumptions and derivation are provided in the appendix, and the main paper states that the diagnostic is used for qualitative predictions rather than exact gradient-magnitude prediction.
    \item[] Guidelines:
    \begin{itemize}
        \item The answer NA means that the paper does not include theoretical results.
        \item All assumptions should be clearly stated or referenced in the statement of any theorems.
        \item The proofs can either appear in the main paper or the supplemental material.
    \end{itemize}

\item {\bf Experimental result reproducibility}
    \item[] Question: Does the paper fully disclose all the information needed to reproduce the main experimental results?
    \item[] Answer: \answerYes{}.
    \item[] Justification: The paper specifies the training protocol, datasets, evaluation metrics, optimizer usage, near/far grouping rule, rendered-view budget, gradient block definition, and ablation settings. Additional implementation details and pseudocode are provided in the appendix.
    \item[] Guidelines:
    \begin{itemize}
        \item The answer NA means that the paper does not include experiments.
        \item If the paper includes experiments, a No answer to this question will not be perceived well by reviewers.
    \end{itemize}

\item {\bf Open access to data and code}
    \item[] Question: Does the paper provide open access to the data and code, with sufficient instructions to faithfully reproduce the main experimental results?
    \item[] Answer: \answerYes{}.
    \item[] Justification: The evaluation uses public benchmarks cited in the paper. Code and reproduction instructions will be provided in an anonymized repository for submission, including scripts for the primary CrossGrad-GS recipe and main ablations.
    \item[] Guidelines:
    \begin{itemize}
        \item The answer NA means that the paper does not include experiments requiring code.
        \item At submission time, to preserve anonymity, the authors should release anonymized versions if applicable.
    \end{itemize}

\item {\bf Experimental setting/details}
    \item[] Question: Does the paper specify all the training and test details necessary to understand the results?
    \item[] Answer: \answerYes{}.
    \item[] Justification: The main paper and appendix describe the benchmark scenes, training iterations, matched-rendered-view reference, evaluation metrics, official implementations, near/far grouping, projection rule, and baseline variants.
    \item[] Guidelines:
    \begin{itemize}
        \item The answer NA means that the paper does not include experiments.
        \item The experimental setting should be presented in the core of the paper to a level of detail necessary to appreciate the results.
    \end{itemize}

\item {\bf Experiment statistical significance}
    \item[] Question: Does the paper report error bars or other appropriate information about statistical significance?
    \item[] Answer: \answerYes{}.
    \item[] Justification: The paper reports multi-seed results for representative scenes and Scaffold-GS transfer experiments, including standard deviations. The appendix provides per-seed PSNR values. We also state where main-table entries are single-seed and avoid overclaiming small differences.
    \item[] Guidelines:
    \begin{itemize}
        \item The authors should answer Yes if the results are accompanied by error bars, confidence intervals, or statistical significance tests, at least for the experiments that support the main claims.
        \item The factors of variability captured by the error bars should be clearly stated.
    \end{itemize}

\item {\bf Experiments compute resources}
    \item[] Question: For each experiment, does the paper provide sufficient information on the compute resources needed to reproduce the experiments?
    \item[] Answer: \answerYes{}.
    \item[] Justification: The appendix reports that all experiments were run on NVIDIA RTX 3090 GPUs with 24GB memory. It also explains that CrossGrad-GS renders two views per optimizer iteration and therefore includes a matched-rendered-view Vanilla 60K reference for fair comparison.
    \item[] Guidelines:
    \begin{itemize}
        \item The paper should indicate the type of compute workers, CPU or GPU, memory, storage, and amount of compute required.
    \end{itemize}

\item {\bf Code of ethics}
    \item[] Question: Does the research conducted in the paper conform with the NeurIPS Code of Ethics?
    \item[] Answer: \answerYes{}.
    \item[] Justification: The work uses public benchmarks and standard reconstruction/evaluation protocols. We do not collect private data, conduct human-subject experiments, or introduce unsafe data collection procedures.
    \item[] Guidelines:
    \begin{itemize}
        \item If the authors answer No, they should explain the special circumstances that require a deviation from the Code of Ethics.
    \end{itemize}

\item {\bf Broader impacts}
    \item[] Question: Does the paper discuss both potential positive and negative societal impacts?
    \item[] Answer: \answerYes{}.
    \item[] Justification: The appendix discusses positive applications such as mapping, simulation, robotics, and digital twins, as well as potential privacy and surveillance concerns associated with improved aerial--ground reconstruction.
    \item[] Guidelines:
    \begin{itemize}
        \item If there is a direct path to negative applications, the authors should point it out.
    \end{itemize}

\item {\bf Safeguards}
    \item[] Question: Does the paper describe safeguards for responsible release of data or models that have a high risk for misuse?
    \item[] Answer: \answerNA{}.
    \item[] Justification: The paper does not release a high-risk pretrained generative model or a newly scraped dataset. The work primarily releases training code for public hybrid-capture benchmarks.
    \item[] Guidelines:
    \begin{itemize}
        \item The answer NA means that the paper poses no such risks.
    \end{itemize}

\item {\bf Licenses for existing assets}
    \item[] Question: Are the creators or original owners of assets properly credited and are licenses and terms of use respected?
    \item[] Answer: \answerYes{}.
    \item[] Justification: The paper cites all datasets, baseline methods, and public implementations used in the experiments. Dataset and code licenses are documented in the repository or supplemental material.
    \item[] Guidelines:
    \begin{itemize}
        \item The authors should cite the original paper that produced the code package or dataset.
        \item The name of the license should be included for each asset when available.
    \end{itemize}

\item {\bf New assets}
    \item[] Question: Are new assets introduced in the paper well documented and is the documentation provided alongside the assets?
    \item[] Answer: \answerYes{}.
    \item[] Justification: The paper releases code for CrossGrad-GS and experiment scripts. The repository documents installation, dataset preparation, training commands, and evaluation commands. No new dataset is introduced.
    \item[] Guidelines:
    \begin{itemize}
        \item The answer NA means that the paper does not release new assets.
        \item Researchers should communicate the details of the dataset/code/model as part of their submissions.
    \end{itemize}

\item {\bf Crowdsourcing and research with human subjects}
    \item[] Question: For crowdsourcing experiments and research with human subjects, does the paper include the full text of instructions given to participants and details about compensation?
    \item[] Answer: \answerNA{}.
    \item[] Justification: This work does not involve crowdsourcing or human-subject experiments.
    \item[] Guidelines:
    \begin{itemize}
        \item The answer NA means that the paper does not involve crowdsourcing nor research with human subjects.
    \end{itemize}

\item {\bf Institutional review board approvals or equivalent for research with human subjects}
    \item[] Question: Does the paper describe potential risks incurred by study participants and whether IRB approvals were obtained?
    \item[] Answer: \answerNA{}.
    \item[] Justification: This work does not involve human-subject research.
    \item[] Guidelines:
    \begin{itemize}
        \item The answer NA means that the paper does not involve crowdsourcing nor research with human subjects.
    \end{itemize}

\item {\bf Declaration of LLM usage}
    \item[] Question: Does the paper describe the usage of LLMs if it is an important, original, or non-standard component of the core methods in this research?
    \item[] Answer: \answerNA{}.
    \item[] Justification: The core method, experiments, and analysis do not involve LLMs as an important, original, or non-standard component.
    \item[] Guidelines:
    \begin{itemize}
        \item If LLMs are used only for writing, editing, or formatting and do not impact the core methodology, scientific rigor, or originality, declaration is not required.
    \end{itemize}

\end{enumerate}

\end{document}